\documentclass[runningheads]{llncs}
\usepackage[T1]{fontenc}
\usepackage{graphicx}
\usepackage{booktabs}
\usepackage{amsmath}
\usepackage{amssymb}
\usepackage{mathtools}
\usepackage[misc]{ifsym}
\usepackage{algorithm}
\usepackage{algorithmic}
\usepackage{float}
\usepackage{bbm}

\newcommand{\restate}[2]{
    \par\vspace{\topsep}
    \noindent \textbf{#1 \ref{#2} (Restated).} \itshape
}
\usepackage{mwe}
\begin{document}

\title{Regret Analysis of Sleeping Competing Bandits}
\author{Shinnosuke Uba \and Yutaro Yamaguchi}
\authorrunning{Shinnosuke Uba and Yutaro Yamaguchi}
\institute{Osaka University, Japan.\\\email{yutaro.yamaguchi@ist.osaka-u.ac.jp}}

\maketitle              %

\begin{abstract}
The Competing Bandits framework is a recently emerging area that integrates multi-armed bandits 
in online learning with stable matching in game theory. While conventional models assume 
that all players and arms are constantly available, in real-world problems, 
their availability can vary arbitrarily over time. In this paper, we formulate this setting 
as \textit{Sleeping Competing Bandits}. To analyze this problem, we naturally extend 
the regret definition used in existing competing bandits and derive regret bounds for the proposed model.
We propose an algorithm that simultaneously achieves an asymptotic regret bound of $\mathrm{O}\left(NK\log T_{i}/\Delta^2\right)$ under reasonable assumptions,
where $N$ is the number of players, $K$ is the number of arms, $T_{i}$ is the number of rounds of each player $p_i$, and $\Delta$ is the minimum reward gap.
We also provide a regret lower bound of $\mathrm{\Omega}\left( N(K-N+1)\log T_{i}/\Delta^2 \right)$ under the same assumptions. 
This implies that our algorithm is asymptotically optimal in the regime where the number of arms $K$ is relatively larger than the number of players $N$.

\end{abstract}

\section{Introduction}
The Multi-Armed Bandit (MAB) is a fundamental online learning paradigm that balances exploration and exploitation. MAB is broadly classified into \textit{adversarial}~\cite{auer1995gambling} and \textit{stochastic}~\cite{Robbins1952SomeAO,thompson1933likelihood} settings. Focusing on the latter, the standard objective is to minimize expected regret. Foundational results in this domain include theoretical lower bounds~\cite{lai1985asymptotically} and classic algorithms like Upper Confidence Bound (UCB)~\cite{auer2002finite}, Thompson Sampling (TS)~\cite{thompson1933likelihood}, and Explore-Then-Commit (ETC)~\cite{garivier2019explore,lattimore2020bandit} (see~\cite{honda2016bandit,lattimore2020bandit,slivkins2024introductionmultiarmedbandits,bubeck2012regret} for comprehensive surveys).

Parallel to bandit problems, stable matching~\cite{gale1962college} is a foundational game-theoretic model for two-sided markets (e.g., players and arms). A matching is \textit{stable} if no pair mutually prefers each other over their current assignments. The Gale--Shapley (GS) algorithm~\cite{gale1962college} finds such a matching and uniquely yields the optimal outcome for the proposing side. This framework has notable real-world applications, such as medical residency matching~\cite{roth1984evolution}.
Stable matching has also been studied in repeated settings~\cite{das2005two,johari2021matching}. 

Recently, the \textit{Competing Bandits}~\cite{liu2020competing} framework has been proposed to bridge these two fields. 
This problem analyzes a two-sided market where the preferences of players over arms are unknown. Through repeated matchings, players observe stochastic rewards and learn their preferences. A distinguishing feature of this model is that arms also possess preferences over players~\cite{boursier2024survey,liu2020competing};
when multiple players select the same arm, a collision occurs, and the arm accepts only the player it prefers most. 
To evaluate the performance, two notions of regret --- player-optimal stable regret and player-pessimal stable regret --- are defined, measuring the difference between the collected rewards and those in a stable matching. 
The goal is to achieve sublinear regret for all players.

Existing literature on Competing Bandits typically assume that all players and arms are available in every round. 
However, when modeling real-world applications --- such as the matching of couriers and orders on food delivery platforms --- it is crucial to address scenarios where availability fluctuates over time.
In the context of multi-armed bandits, this issue is addressed in the \textit{Sleeping Bandits}~\cite{kleinberg2010regret}.
A key feature of this setting is that it places no assumptions on the stochasticity of availability; rather, the subsets of available arms are determined arbitrarily by the environment in each round.
Extending this concept to Competing Bandits, in this paper, we introduce the \textit{Sleeping Competing Bandits}. 

We aim to define a metric of regret suitable for this dynamic setting, propose an algorithm that achieves a sublinear regret upper bound, and derive a fundamental lower bound on the regret incurred by any algorithm.

\subsection{Main Contributions}
Our main contributions are summarized as follows:

\begin{itemize}
    \item \textbf{Formulation of Sleeping Competing Bandits:} 
    We formulate a new problem setting where the availability of both players and arms can vary over time, extending the standard competing bandits framework. 
    We formally define the player-optimal stable regret $\overline{R}_i(T_i)$ and player-pessimal stable regret $\underline{R}_i(T_i)$ for this setting.
    Notably, in the special case where all arms and players are always available and the arms' preference rankings are fixed, 
    our definitions reduce to the standard competing bandits framework~\cite{liu2020competing}.

    \item \textbf{Regret Lower Bounds:} 
    We analyze the fundamental hardness of the problem. First, without specific assumptions, any policy cannot achieve strictly sublinear regret (Theorem \ref{thm:generallowerbound}). Second, even under reasonable assumptions, we derive a lower bound of 
    $\overline{R}_i(T_i) \geq \underline{R}_i(T_i) = \mathrm{\Omega}(N(K-N+1)\log T_i / \Delta^2)$ (Theorem \ref{thm:lowerbound}), where $\Delta$ represents the minimum reward gap.
    This bound is larger than that of the standard setting by a factor of $K$~\cite{sankararaman2021dominate}.

    \item \textbf{Regret Upper Bounds:} 
    We propose an algorithm that naturally extends the Centralized UCB algorithm~\cite{liu2020competing} to the sleeping setting.
    We prove that it achieves a player-pessimal stable regret $\underline{R}_i(T_i) = \mathrm{O}(NK\log T_i / \Delta^2)$ 
    under the same assumptions as in Theorem~\ref{thm:lowerbound} (Theorem \ref{thm:pessimal_upper_bound}). 
    This establishes the asymptotic optimality of our method in the regime where the number of arms $K$ is relatively larger than the number of players $N$. 
    Additionally, for the player-optimal stable regret $\overline{R}_i(T_i)$, we propose another algorithm that achieves an upper bound of $\mathrm{O}(NK^2\log T_i / \Delta^2)$ (Theorem \ref{thm:optimal_upper_bound}).
    This algorithm adaptively switches between exploration and exploitation rounds based on the criterion proposed in \cite{kong2023player}.
\end{itemize}

\section{Related Work}
\subsection{Single Player Multi-Armed Bandit}
While the standard MAB assumes that all arms are available in every round, several settings account for arm unavailability. 
Examples include \textit{sleeping bandits}~\cite{kleinberg2010regret}, where only a subset of arms is available in each round; 
\textit{mortal bandits}~\cite{chakrabarti2008mortal}, where arms become permanently unavailable after a certain duration; 
and \textit{blocking bandits}~\cite{basu2019blocking}, where a pulled arm becomes unavailable for a fixed period. 
In this study, we focus on the sleeping bandit setting. 

Another relevant extensions are \textit{Combinatorial Bandit} framework~\cite{chen2013combinatorial} and its sleeping extensions~\cite{chen2018contextual,li2019combinatorial} allow players to select a subset of arms (a super arm) in each round. Crucially, these models treat arms as passive resources. In contrast, our \textit{Sleeping Competing Bandits} framework addresses a two-sided market where arms are active entities with distinct preferences. Thus, our objective shifts from merely maximizing cumulative rewards to achieving a stable matching under dynamic availability.

\subsection{Multi-Player Bandits}
Extending the single-player model, the Multi-Player Bandits (MPB) problem involves multiple players selecting arms. As noted by Boursier and Perchet~\cite{boursier2024survey}, this framework was motivated by cognitive radio networks and initially proposed by Mitola and Maguire~\cite{mitola2002cognitive}. MPB is generally categorized into \textit{decentralized} and \textit{centralized} settings; our work aligns with the latter, where a central decision-maker coordinates selection.

While some MPB studies address dynamic environments~\cite{avner2014concurrent,boursier2019sic}, they typically assume constant arm availability, allowing only the player set to vary. Furthermore, their optimization goals and regret definitions differ from ours.

\subsection{Competing Bandits}\label{sec:related_work_competing_bandits}
Finally, we discuss \textit{Competing Bandits}, which constitutes the primary problem setting of this study. This model was first proposed by Liu, Mania, and Jordan~\cite{liu2020competing}. 
As noted in the survey on MPB~\cite{boursier2024survey}, the distinguishing feature of this model is that arms also possess preferences over players; when a collision occurs (i.e., multiple players select the same arm), only the player most preferred by the arm receives the reward. 
While the initial work~\cite{liu2020competing} considered a centralized setting coordinated by a platform, subsequent research has focused on decentralized settings where players independently select arms using the same algorithm~\cite{sankararaman2021dominate,basu2021beyond,zhang2022matching,kong2023player,kong2024improved,kong2022thompson,maheshwari2022decentralized}.
Despite the growing interest in decentralized settings, the centralized approach remains crucial for applications where a central platform (e.g., ride-hailing or food delivery apps) assigns tasks to agents.
It is worth noting that even within centralized or decentralized categories, problem settings --- such as what can be observed --- vary depending on the specific application.

\section{Proposed Model and Methods}\label{chap:Proposed model and methods}
\subsection{Basic Notation and Definitions}
Let $T$ be the number of rounds, and let $\mathcal{T} = \{1, 2, \dots, T\}$ denote the set of rounds.
Let $\mathcal{P} = \{p_1, p_2, \dots, p_N\}$ denote the set of $N$ players and $\mathcal{A} = \{a_1, a_2, \dots, a_K\}$ denote the set of $K$ arms.
We assume that $N\leq K$ without loss of generality.
Let $[n]$ denote the set $\{1, 2, \dots, n\}$ for a positive integer $n$.

At each round $t \in \mathcal{T}$, let $\mathcal{P}_t \subseteq \mathcal{P}$ and $\mathcal{A}_t \subseteq \mathcal{A}$ denote the sets of available players and arms, respectively.
For each player $p_i$, we define the set of rounds where the player is available as $H_i \coloneqq \{t \in \mathcal{T} \mid p_i \in \mathcal{P}_t\}$, and let $T_i = |H_i|$ be the total number of such rounds.
Let $h_i$ be the sequence constructed by sorting all elements of $H_i$ in ascending order. We denote by $t_i$ the local round index for player $p_i$, such that $h_i(t_i)$ corresponds to the global round index of the $t_i$-th round in which player $p_i$ is available.

At each round $t \in \mathcal{T}$, each arm $a_j\in \mathcal{A}_t$ possesses a strict preference ranking over the set of players $\mathcal{P}_t$, which we denote by $\succ_{j,t}$.
As a notable feature of our model, we emphasize that these preference rankings can vary across rounds.
However, it is assumed that the current preference ranking is fully known to the platform at the beginning of each round $t$.
Conversely, players' preferences are represented by mean rewards $\{\mu_{i,j}\}$, which are fixed and unknown to the platform; we assume $\mu_{i,j} \in (0,1)$ for all $i \in [N]$ and $j \in [K]$.
We say player $p_i$ prefers arm $a_j$ over $a_{j'}$ and denote it by $a_j \succ_i a_{j'}$ if $\mu_{i,j} > \mu_{i,j'}$.
To ensure strict preference rankings among the available arms and avoid ties at each round, distinct mean rewards are assumed for any pair of arms simultaneously present in the same round. Specifically, for any $t \in \mathcal{T}$ and any distinct pair $a_j, a_{j'} \in \mathcal{A}_t$, it holds that $\mu_{i,j} \neq \mu_{i,j'}$.

Furthermore, each arm $a_j \in \mathcal{A}$ has a capacity $c_j(t)$ for each round $t$, indicating the maximum number of players it can be matched with simultaneously in that round; this can also vary across rounds.

At each round $t \in \mathcal{T}$, the platform assigns an arm index $m_t(i) \in \{j \mid a_j \in \mathcal{A}_t\} \cup \{0\}$ to each available player $p_i \in \mathcal{P}_t$, where $m_t(i) = j$ implies that player $p_i$ is matched to arm $a_j$, and $m_t(i) = 0$ indicates that player $p_i$ is unmatched (i.e., receives no assignment).
Let $m_t^{-1}(j) \coloneqq \{i \mid p_i \in \mathcal{P}_t \text{ and } m_t(i) = j\}$ denote the set of indices of players matched to arm $a_j$ at round $t$.
The matching must satisfy the capacity constraint $|m_t^{-1}(j)| \leq c_j(t)$ for all $a_j \in \mathcal{A}_t$.
Subsequently, the platform observes a stochastic reward $r_{i, m_t(i)}(t)$ for each matched player-arm pair, which is drawn from the Bernoulli distribution with mean $\mu_{i, m_t(i)}$.
If $p_i$ is unmatched, the reward is zero, and we virtually define $\mu_{i, 0} \coloneqq 0$.

\subsection{Definitions of Regret}
Before introducing the regret definitions, we formally define the stability of a matching in our setting.

\begin{definition}[Stable Matching~\cite{gale1962college}]\label{def:stable_matching}
A matching $m_t$ at round $t$ is \emph{stable} if there exists no \emph{blocking pair}, defined as follows: %
Since $\mu_{i,j} > 0$, being matched to any arm is strictly preferred to being unmatched for players. 
A pair $(p_i, a_j) \in \mathcal{P}_t \times \mathcal{A}_t$ is a \emph{blocking pair} for $m_t$ if both of the following conditions hold:
\begin{enumerate}
    \item Player $p_i$ prefers $a_j$ over their current assignment by $m_t$ (i.e., $\mu_{i,j} > \mu_{i,m_t(i)}$).
    \item Arm $a_j$ either has available capacity (i.e., $|m_t^{-1}(j)| < c_j(t)$), or prefers $p_i$ over at least one player currently matched to it (i.e., there exists $i' \in m_t^{-1}(j)$ such that $p_i \succ_{j,t} p_{i'}$).
\end{enumerate}
\end{definition}

Similar to the competing bandits~\cite{liu2020competing}, we define the player-optimal stable regret $\overline{R}_i(T_i)$ and the player-pessimal stable regret $\underline{R}_i(T_i)$ in our problem setting.
It is well-known that there always exist two unique stable matchings $\overline{m}_t$ and $\underline{m}_t$ (which may coincide with each other) such that $a_{\overline{m}_t(i)}  \succeq_i a_{m_t(i)} \succeq_i a_{\underline{m}_t(i)}$ holds for any stable matching $m_t$ and any player $p_i$.
We call $\overline{m}_t$ and $\underline{m}_t$ the \emph{player-optimal} and \emph{player-pessimal} stable matchings at round $t$, respectively.
These matchings correspond to the outputs of the player-proposing and arm-proposing GS algorithms~(cf.~Algorithm~\ref{alg:gs_algorithm} in Appendix~\ref{sec:GS_Algorithm}), given the true preferences of players $\mathcal{P}_t$ 
and arms $\mathcal{A}_t$, and the arm capacities.

\begin{definition}[Player-Optimal/Pessimal Stable Regrets]\label{def:stable_regret_sleeping}
The \emph{player-optimal} and \emph{player-pessimal stable regrets} of player $p_i$ over $T$ rounds are defined as
\begin{align}
\overline{R}_i(T_i) &\coloneqq \sum_{t_i=1}^{T_i} \left( \mu_{i, \overline{m}_{h_i(t_i)}(i)} - \mathbb{E}\left[r_{i, m_{h_i(t_i)}(i)}(h_i(t_i))\right] \right),\\
\underline{R}_i(T_i) &\coloneqq \sum_{t_i=1}^{T_i} \left( \mu_{i, \underline{m}_{h_i(t_i)}(i)} - \mathbb{E}\left[r_{i, m_{h_i(t_i)}(i)}(h_i(t_i))\right] \right).
\end{align}
\end{definition}

Notably, in the special case where all arms and players are always available and the arms' preference rankings are fixed across rounds, our regret definitions reduce to the standard competing bandits framework presented in~\cite{liu2020competing}.

The platform's goal is to achieve (strictly) sublinear regret for all players.
We define the desirable property for the proposed algorithm.

\begin{definition}[$\alpha$-Consistency~\cite{salomon2013lower}]
\label{def:consistency_sleeping}
Let $\alpha \in [0, 1)$.
A policy is said to be \emph{$\alpha$-consistent} if, for any set of underlying probability distributions $\{P_{i,j}\}_{i \in [N], j \in [K]} \in \mathcal{D}^{N \times K}$ and any $c > \alpha$, the player-optimal regret $\overline{R}_i(T_i)$ or the player-pessimal regret $\underline{R}_i(T_i)$ is bounded by $\mathrm{o}(T_i^c)$ for all players $p_i \in \mathcal{P}$.
\end{definition}

Note that the case where $\alpha = 0$ corresponds to the usual consistency, and if $\alpha \ge 1$, it becomes meaningless because the regrets are always $\mathrm{O}(T_i)$ by definition.
Also, if $\beta < \alpha$, the $\beta$-consistency implies the $\alpha$-consistency by definition.

\subsection{Proposed Methods}
In this section, we propose an algorithm for the sleeping competing bandits problem introduced in this chapter.
Before presenting the proposed algorithm, we introduce the Upper Confidence Bound (UCB) and the Lower Confidence Bound (LCB), which are utilized in our method.

\subsubsection{UCB and LCB}
In the proposed algorithm, let $T_{i,j}(t_i)$ $(t_i = 1, 2, \dots, T_i)$ be the count of matches between player $p_i \in \mathcal{P}$ and arm $a_j \in \mathcal{A}$ prior to round $t_i$.
At each round $t_i$, the algorithm computes the UCB \eqref{eq:def_ucb} and the LCB \eqref{eq:def_lcb} for player $p_i$.
To ensure that every arm is explored at least once, we set $\mathrm{UCB}_{i,j}(t_i) = +\infty$ and $\mathrm{LCB}_{i,j}(t_i) = -\infty$ if $T_{i,j}(t_i) = 0$.
For $T_{i,j}(t_i) > 0$, they are defined as follows:
\begin{align}
    \mathrm{UCB}_{i,j}(t_i) &\coloneqq \hat{\mu}_{i,j}(t_i) +  \sqrt{\frac{\log t_i}{T_{i,j}(t_i)}}, \label{eq:def_ucb} \\
    \mathrm{LCB}_{i,j}(t_i) &\coloneqq \hat{\mu}_{i,j}(t_i) -  \sqrt{\frac{\log t_i}{T_{i,j}(t_i)}}, \label{eq:def_lcb}
\end{align}
where $\hat{\mu}_{i,j}(t_i)$ represents the empirical mean reward of arm $a_j$ observed by player $p_i$ prior to round $t_i$.

\subsubsection{Awake Centralized UCB Algorithm}
Here, we present the Awake Centralized UCB Algorithm (AC-UCB). This is a natural extension of the Centralized UCB algorithm proposed in \cite{liu2020competing} to the sleeping competing bandits setting.

In each round $t \in \mathcal{T}$, the platform performs the following:
\begin{algorithm}[H]
\caption{AC-UCB Algorithm}\label{alg:ac_ucb_algorithm}
\mbox{}\\
\textbf{Input:} Sets of available players $\mathcal{P}_t$ and arms $\mathcal{A}_t$, along with the arms' preference rankings and capacities $\{(\succ_{j,t}, c_j(t))\}_{a_j \in \mathcal{A}_t}$.
\begin{enumerate}
  \item For each player $p_i \in \mathcal{P}_t$, construct a preference ranking $\sigma_i(t_i)$ over the available arms $\mathcal{A}_t$ by sorting them according to the UCB indices $\mathrm{UCB}_{i,j}(t_i)$ in descending order. Ties are broken arbitrarily.
  \item Execute the player-proposing GS algorithm on the stable matching instance $(\mathcal{P}_t, \mathcal{A}_t, \{\sigma_i(t_i)\}_{p_i \in \mathcal{P}_t}, \{(\succ_{j,t}, c_j(t))\}_{a_j \in \mathcal{A}_t})$ and obtain matching $m_t$.
  \item For each player $p_i \in \mathcal{P}_t$ who is matched to an arm (i.e., $m_t(i) \neq 0$), observe the stochastic reward $r_{i, m_t(i)}(t)$, update the empirical mean $\hat{\mu}_{i, m_t(i)}(t_i+1)$ and count $T_{i, m_t(i)}(t_i+1)$. The parameters for unmatched players remain unchanged.
\end{enumerate}
\end{algorithm}

\subsubsection{Awake Centralized Explore-Then-Gale--Shapley Algorithm}
Here, we propose the Awake Centralized Explore-Then-Gale--Shapley (AC-ETGS) Algorithm. This algorithm separates the process into \textit{exploration rounds}, where random matching is performed, and \textit{exploitation rounds}, where the player-optimal stable matching is computed.
The criterion for switching between exploration and exploitation is based on the Explore-Then-Gale--Shapley Algorithm proposed in \cite{kong2023player}.

In each round $t \in \mathcal{T}$, the platform performs the following:
\begin{algorithm}[H]
\caption{AC-ETGS Algorithm}\label{alg:ac_etgs_algorithm}
\mbox{}\\
\textbf{Input:} Sets of available players $\mathcal{P}_t$ and arms $\mathcal{A}_t$, along with the arms' preference rankings and capacities $\{(\succ_{j,t}, c_j(t))\}_{a_j \in \mathcal{A}_t}$.
\begin{enumerate}
  \item Calculate $\mathrm{UCB}_{i,j}(t_i)$ and $\mathrm{LCB}_{i,j}(t_i)$ for every pair of $p_i \in \mathcal{P}_t$ and $a_j \in \mathcal{A}_t$. %
  
  \item Check if the preference ordering for every player is determined with high confidence. Specifically, check if for every player $p_i \in \mathcal{P}_t$, there exists a permutation $\sigma_i$ of available arms such that for all $k \in [|\mathcal{A}_t|-1]$, the condition $\mathrm{LCB}_{i, \sigma_i(k)}(t_i) > \mathrm{UCB}_{i, \sigma_i(k+1)}(t_i)$ holds.
  
  \item If the condition in Step 2 holds (Exploitation Round):
  \begin{itemize}
      \item Construct the preference ranking $\sigma_i(t_i)$ for each player by sorting the available arms according to $\mathrm{UCB}_{i,j}(t_i)$ in descending order.
      \item Execute the player-proposing GS algorithm on the stable matching instance $(\mathcal{P}_t, \mathcal{A}_t, \{\sigma_i(t_i)\}_{p_i \in \mathcal{P}_t}, \{(\succ_{j,t}, c_j(t))\}_{a_j \in \mathcal{A}_t})$ and obtain matching $m_t$.
  \end{itemize}
  
  \item Otherwise (Exploration Round):
  \begin{itemize}
      \item Construct a random matching $m_t$ consisting of $\min(|\mathcal{P}_t|, |\mathcal{A}_t|)$ pairs uniformly at random, treating the capacity of every arm $a_j \in \mathcal{A}_t$ as 1.
  \end{itemize}

  \item For each player $p_i \in \mathcal{P}_t$ who is matched to an arm (i.e., $m_t(i) \neq 0$), observe the stochastic reward $r_{i, m_t(i)}(t)$, update the empirical mean $\hat{\mu}_{i, m_t(i)}(t_i+1)$ and count $T_{i, m_t(i)}(t_i+1)$. The parameters for unmatched players remain unchanged.
\end{enumerate}
\end{algorithm}

\section{Regret Lower Bounds}\label{sec:lower}
This section investigates the regret lower bounds for the sleeping competing bandits model.
First, we show that without any specific assumptions, there is no $\alpha$-consistent policy for any $\alpha \in [0, 1)$. %
\begin{theorem}\label{thm:generallowerbound}
 For any policy and any constant $c \in (0, 1)$, there exists a problem instance (a collection of reward distributions) such that the player-optimal stable regret $\overline{R}_i(T_i)$ and player-pessimal stable regret $\underline{R}_i(T_i)$ for some player $p_i$ satisfies:
 \[
     \overline{R}_i(T_i)\geq \underline{R}_i(T_i) = \omega(T_i^c).
 \]
\end{theorem}
\begin{proof}
Here we provide a sketch of the proof; the complete proof is available in Appendix~\ref{ap:general_lowerbound}.
We construct a hard instance where a fixed target player faces a sequence of fresh opponents.
Every $L$ rounds, a new competing player arrives, participates alongside the target player for this duration, and then leaves.
There are only two active arms, and both of them prefer the current competitor over the target player.
Thus, the target player's optimal arm depends entirely on the competitor's choice.
Since each competitor faces a fresh bandit instance, if $L = \omega(1)$, they must incur $\mathrm{\Omega}(\log L)$ exploration steps to identify their preferred arm for bounding their regret by $\mathrm{O}(L^c)$.
Crucially, these exploration rounds by the competitor correspond to the target player's optimal arm, causing the target player to be blocked.
Since this exploration occurs for every new competitor, the target player suffers regret $\mathrm{\Omega}((T/L) \cdot \log L) = \mathrm{\Omega}(T^c \log T) = \omega(T^c)$ if we take $L = \mathrm{\Theta}(T^{1-c}) = \omega(1)$. \qed
\end{proof}

Next, we present a more refined lower bound under appropriate assumptions.
Let $\Delta$ be the minimum difference in mean rewards between any distinct pair of available arms across all players and rounds, defined as follows:
$$\Delta = \min_{p_i \in \mathcal{P}} \min_{t \in \mathcal{T}: p_i \in \mathcal{P}_t} \min_{\substack{a_j, a_{j'} \in \mathcal{A}_t \\ j \neq j'}} |\mu_{i,j} - \mu_{i,j'}|.$$
Recall that we assume $N \le K$.
\begin{theorem}\label{thm:lowerbound}
Assume $T_i \approx T_j$ for all players $p_i, p_j \in \mathcal{P}$ and $K=\mathrm{O}(\log T_i)$. 
Then, for any $\alpha \in [0, 1)$ and any $\alpha$-consistent policy, there exists a problem instance such that the player-optimal
stable regret $\overline{R}_i(T_i)$ and player-pessimal stable regret $\underline{R}_i(T_i)$ for some player $p_i$ satisfy:
\[
\overline{R}_i(T_i)\geq \underline{R}_i(T_i) = \mathrm{\Omega}\left( N(K-N+1)\log T_{i}/\Delta^2 \right).
\]
\end{theorem}

\begin{proof}
    To establish the lower bound in Theorem \ref{thm:lowerbound}, we employ a standard change-of-measure argument. We construct a reference instance, denoted as \textbf{Instance 1}, and a family of alternative instances $\{\textbf{Instance }(i, k)\}_{i, k}$, one for each competing player $p_i$ and each variable arm $a_k$.

    \textbf{Common Settings.} Consider a set of $N$ players $\mathcal{P} = \{p_1, \dots, p_N\}$ and a set of $K$ arms $\mathcal{A} = \{a_1, \dots, a_K\}$, with $N \le K$. All players are available in all rounds $t \in \mathcal{T}$.
    We partition the set of arms into two subsets:
    \begin{itemize}
        \item \textbf{Fixed Arms ($\mathcal{A}_{\text{fix}}$):} Let $\mathcal{A}_{\text{fix}} = \{a_1, \dots, a_{N-1}\}$. These arms are available in every round $t \in \mathcal{T}$.
        \item \textbf{Variable Arms ($\mathcal{A}_{\text{var}}$):} Let $\mathcal{A}_{\text{var}} = \{a_N, \dots, a_K\}$. In each round $t$, exactly one arm from $\mathcal{A}_{\text{var}}$ is available.
    \end{itemize}
    We assume unit capacity for all arms, i.e., $c_j(t) = 1$ for all $j, t$.

    Let $\mathcal{A}_t \subseteq \mathcal{A}$ be the set of available arms at round $t$. By construction, $\mathcal{A}_t = \mathcal{A}_{\text{fix}} \cup \{a_{j_t}\}$ for some $a_{j_t} \in \mathcal{A}_{\text{var}}$.
    For each variable arm $a_j \in \mathcal{A}_{\text{var}}$, let $\mathcal{T}_j = \{ t \in \mathcal{T} \mid a_j \in \mathcal{A}_t \}$ denote the rounds where it is available. We ensure a balanced schedule such that $|\mathcal{T}_j| \ge \lfloor T / |\mathcal{A}_{\text{var}}| \rfloor = \lfloor T / (K-N+1) \rfloor$ for all $a_j \in \mathcal{A}_{\text{var}}$.

    For any round $t$ and any available arm $a_j$, the arm's preference over players is fixed and strictly hierarchical based on player indices:
    \begin{equation}
        p_1 \succ_{j} p_2 \succ_{j} \cdots \succ_{j} p_N.
        \notag %
    \end{equation}
    This ordering places player $p_N$ at the lowest priority, putting them at a disadvantage in any contention.

    \textbf{Unknown Player Preferences.} We set the mean rewards $\mu_{i,j}$ such that the gap parameter $\Delta > 0$ is sufficiently small (relative to $0.1$), ensuring that the difference between any pair of distinct arm values is at least $\Delta$ in the reference instance.

    \textbf{Instance 1 (Reference Instance).} In this instance, preferences are designed such that the ``victim'' player $p_N$ prefers any variable arm, while each competing player $p_i$ ($i < N$) prefers their specific fixed arm $a_i$.
    Specifically, for player $p_N$, we set $\mu_{N,j} \ge 0.6$ if $a_j \in \mathcal{A}_{\text{var}}$ and $\mu_{N,j} \le 0.5$ otherwise.
    For each competing player $p_i$ ($i < N$), the reward structure is defined as follows:
    $$\mu_{i,j} =
    \begin{cases}
         0.5  & \text{if } j = i \quad (\text{Optimal Fixed Arm}),\\
         0.5-\Delta & \text{if } a_j \in \mathcal{A}_{\text{var}},\\
        < 0.5-2\Delta & \text{otherwise}.
    \end{cases}$$
    Consequently, in any round $t \in \mathcal{T}_k$ where a variable arm $a_k$ is available, the unique stable matching is $\underline{m}_t = \{(p_N, a_k)\} \cup \{(p_j, a_j) \mid 1 \le j < N \}$, where every player is matched with their most preferred available arm.

    \textbf{Instance $(i, k)$ (Alternative Instance).}
    For a target competing player $p_i$ ($i < N$) and a target variable arm $a_k \in \mathcal{A}_{\text{var}}$, we construct an alternative instance by flipping $p_i$'s preference to favor $a_k$ over $a_i$.
    Let $\epsilon$ be a small constant satisfying $0 < \epsilon < 0.1\Delta$. We increase the mean reward for the pair $(p_i, a_k)$ to $\mu'_{i,k} = 0.5 + \epsilon$, while keeping all other mean rewards identical to those in Instance 1.
    Under these modified preferences, the unique stable matching in round $t \in \mathcal{T}_k$ shifts to $\underline{m}'_t = \{(p_i, a_k), (p_N, a_i)\} \cup \{(p_j, a_j) \mid 1 \le j < N, \ j \neq i\}$. 
    Crucially, in this matching, the victim $p_N$ is displaced from the variable arm $a_k$ and forced to match with the fixed arm $a_i$.
    
    \textbf{Regret Analysis.} We analyze the regret of player $p_N$ in Instance 1.
    Ideally, in Instance 1, any competing player $p_i$ should select $a_i$. However, to distinguish Instance 1 from Instance $(i,k)$, $p_i$ must explore $a_k$.

    Let $P_{i,k}(1)$ and $P_{i,k}(i,k)$ be the reward distributions observed by player $p_i$ (when pulling arm $a_k$) in Instance 1 and Instance $(i,k)$, respectively.
    If a policy is consistent (in particular, achieves sublinear regret on both Instance 1 and Instance $(i,k)$), 
    player $p_i$ must select $a_k$ a sufficient number of times in Instance 1 because the distributions of arm $a_k$ for $p_i$ 
    are close in KL-divergence between the two instances (see Appendix~\ref{sec:KL_divergence} for more details).

    Whenever $p_i$ selects $a_k$ in Instance 1, $p_N$ is blocked and incurs a regret at least $0.1$.
    Thus, the total regret of $p_N$ is lower bounded by summing the expected number of suboptimal selections by each $p_i$:
    \begin{align}
    \underline{R}_N(T_N) &= \sum_{t=1}^{T_N} \left( \mu_{N, \underline{m}_{t}(N)} - \mathbb{E}[r_{N, m_{t}(N)}(t)] \right) \notag  \\
    &\geq  0.1 \sum_{i=1}^{N-1} \sum_{k=N}^K \mathbb{E} [T_{i, k}(|\mathcal{T}_k|)] \label{eq:4.8} \\
    &\geq  0.1 \sum_{i=1}^{N-1} \sum_{k=N}^K\frac{(1-c)\log |\mathcal{T}_k|-\log(2C) + \log(\epsilon/4)}{D_{\mathrm{KL}}(P_{i,k}(1) \| P_{i,k}(i,k))} \label{eq:4.9} \\
    &=\mathrm{\Omega}(N(K-N+1)\log T_N/\Delta^2), \label{eq:4.10}
    \end{align}
    where $T_{i, k}(|\mathcal{T}_k|) = \sum_{t \in \mathcal{T}_k} \mathbbm{1}\{m_{t}(i) = a_k\}$ is the number of times $p_i$ selects $a_k$. 
    \eqref{eq:4.9} follows from a standard change-of-measure argument with $\underline{R}_k(|\mathcal{T}_k|) = \mathrm{O}(|\mathcal{T}_k|^c)$ for $c \coloneqq (1 + \alpha) / 2 \in (\alpha, 1)$ by the $\alpha$-consistency of the policy.
    The detailed derivations are deferred to Appendix~\ref{ap:lowerbound}. \qed

\end{proof}

\section{Regret Upper Bounds}\label{sec:upper}
In this section, we analyze the performance of the proposed methods AC-UCB (Algorithm~\ref{alg:ac_ucb_algorithm}) and AC-ETGS (Algorithm~\ref{alg:ac_etgs_algorithm}).
We first present the player-pessimal stable regret $\underline{R}_i(T_i)$ (Definition~\ref{def:stable_regret_sleeping}) upper bound for the AC-UCB algorithm and then the player-optimal stable regret $\overline{R}_i(T_i)$ (Definition~\ref{def:stable_regret_sleeping}) upper bound for the AC-ETGS algorithm. 

Let $\Delta_{i,\max} = \max_{a_j \in \mathcal{A}} \mu_{i, j}$ be the maximum regret that may be suffered by player $p_i$ in any round.
\begin{theorem}\label{thm:pessimal_upper_bound}
  Assume $T_i \approx T_j$ for all players $p_i, p_j \in \mathcal{P}$ and $K=\mathrm{O}(\log T_i)$. 
When using Algorithm~\ref{alg:ac_ucb_algorithm}, the player-pessimal stable regret for any player $p_i$ is bounded as
\begin{align*}
    \underline{R}_i(T_i) =\mathrm{O}(NK\log T_{i}/\Delta^2).
\end{align*}
\end{theorem}
\begin{proof}

We fix any player $p_i$ and focus on their player-pessimal stable regret. 
Let $\mathcal{S}_{h_i(t_i)}$ denote the set of all stable matchings at round $h_i(t_i)$ --- where player $p_i$ participates --- determined by the set of available players and arms, the arm capacities, and the true preferences of both sides. 
In this setting, the player-pessimal stable regret is upper-bounded by the number of times the platform's assignment does not belong to $\mathcal{S}_{h_i(t_i)}$:

\begin{align}
    \underline{R}_i(T_i) 
    &= \sum_{t_i =1}^{T_i} \mathbb{E}\left[ \mu_{i, \underline{m}_{h_i(t_i)}(i)} - r_{i, m_{h_i(t_i)}(i)}(h_i(t_i)) \right] \nonumber \\
    &\leq \Delta_{i,\max} \cdot \mathbb{E}\left[ \sum_{t_i=1}^{T_i} \mathbbm{1}\{ m_{h_i(t_i)} \notin \mathcal{S}_{h_i(t_i)} \} \right]. \label{eq:5.1}
\end{align}

Following the approach in~\cite{liu2020competing}, we bound the expected number of unstable assignments in \eqref{eq:5.1} by analyzing the existence of blocking triplets.

Unlike the standard competing bandits setting, in the sleeping setting, the number of available players $\mathcal{P}_t$ is not necessarily less than or equal to the total capacity of arms $\mathcal{A}_t$ in each round.
Therefore, we extend the definition of a blocking triplet to account for unmatched players and arm capacities.
Let $a_0 = \emptyset$ denote the virtual arm representing unmatched, where each unmatched player $p_i$ (with $m(i) = 0$) is regarded as being matched with this $a_0$ with mean reward $\mu_{i,0} = 0$; we assume $a_0 \in \mathcal{A}_t$ for all $t \in [T]$.

\begin{definition}[Blocking Triplet (Extended)]\label{def:blocking_triplet_extended}
    A triplet $(p_j, a_k, a_{k'})$, where $k \in [K]$ and $k' \in [K] \cup \{0\}$, is defined as a \emph{blocking triplet} for a matching $m$ if $p_j$ is matched with $a_{k'}$ (i.e., $m(j)=k'$) and the pair $(p_j, a_k)$ forms a blocking pair (cf.~Definition~\ref{def:stable_matching}).
    Specifically, this requires both of the following conditions:
    \begin{itemize}
        \item Player $p_j$ prefers $a_k$ over their current assignment $a_{k'}$ (i.e., $\mu_{j,k} > \mu_{j,k'}$).
        \item Arm $a_k$ has available capacity, or $a_k$ prefers $p_j$ over its least-preferred current partner.
    \end{itemize}
\end{definition}

By definition, a matching $m$ is not stable if and only if there exists at least one blocking triplet. 
Using this extended definition, we obtain:
\begin{align}
  &\mathbb{E}\left[ \sum_{t_i=1}^{T_i} \mathbbm{1}\{ m_{h_i(t_i)} \notin \mathcal{S}_{h_i(t_i)} \} \right] \nonumber \\
  &= \mathbb{E}\left[ \sum_{t_i=1}^{T_i} \mathbbm{1}\left\{ \exists (p_j, a_k, a_{k'}) \text{: a blocking triplet for } m_{h_i(t_i)} \right\} \right], \label{eq:blocking_existence}
\end{align}
where the existence is taken over $p_j \in \mathcal{P}_{h_i(t_i)}$, $a_k \in \mathcal{A}_{h_i(t_i)}$, and $a_{k'} \in \mathcal{A}_{h_i(t_i)} \cup \{\emptyset\}$.

Let $L_{j,k,k'}(T_j)$ denote the total number of rounds where player $p_j$ is matched with arm $a_{k'}$ while $(p_j, a_k, a_{k'})$ constitutes a blocking triplet for the matching $m_{h_j(t_j)}$.
We define $\mathcal{W}_{j,k}$ and $\mathcal{W}'_{j,k}$ as follows:
\begin{align*}
\mathcal{W}_{j,k} &\coloneqq \left\{ k' \in [K] \cup \{0\} \mid \mu_{j,k'} < \mu_{j,k} \land \exists t \in H_j \text{ s.t. } \{a_k, a_{k'}\} \subseteq \mathcal{A}_t \right\},\\
\mathcal{W}'_{j,k} &\coloneqq \left\{ k' \in [K] \mid \mu_{j,k'} < \mu_{j,k} \land \exists t \in H_j \text{ s.t. } \{a_k, a_{k'}\} \subseteq \mathcal{A}_t \right\}.
\end{align*}
Using the union bound, \eqref{eq:blocking_existence} is upper-bounded as follows:
\begin{align} &\mathbb{E}\left[ \sum_{t_i=1}^{T_i} \mathbbm{1}\left\{ \exists (p_j, a_k, a_{k'}) \text{: a blocking triplet for } m_{h_i(t_i)} \right\} \right] \nonumber \\
  &\leq \sum_{j=1}^N \sum_{k=1}^K \sum_{k' \in \mathcal{W}_{j,k}} \mathbb{E}\left[ L_{j,k,k'}(T_j) \right] \label{eq:5.4} \\
  &= \sum_{j=1}^N \sum_{k=1}^K \sum_{k' \in \mathcal{W}'_{j,k}} \mathbb{E}\left[ L_{j,k,k'}(T_j) \right]. \label{eq:5.5}
\end{align}
The transition from~\eqref{eq:5.4} to~\eqref{eq:5.5} follows from the fact that the AC-UCB algorithm (Algorithm~\ref{alg:ac_ucb_algorithm}) eliminates the case where $a_{k'} = \emptyset$ (see Lemma~\ref{lem:output of GS} in Appendix~\ref{ap:pessimal} for the complete proof).

Next, we bound $\mathbb{E}\left[ L_{j,k,k'}(T_j) \right]$ for each triplet $(p_j, a_k, a_{k'})$ where $\mu_{j,k'} < \mu_{j,k}$.
Let $\Delta_{j,k,k'} = \mu_{j,k} - \mu_{j,k'} > 0$ be the gap between the mean rewards of arms $a_k$ and $a_{k'}$ for player $p_j$, defined for any pair that is simultaneously available in a round involving $p_j$.
A necessary condition for $(p_j, a_k, a_{k'})$ to be a blocking triplet (specifically, for $p_j$ not to propose to $a_k$ or to prefer $a_{k'}$ based on indices) is that the estimated UCB index of the suboptimal arm $a_{k'}$ exceeds that of the optimal arm $a_k$.
Thus, in any round $t_j$ where such a blocking triplet exists, the inequality $\text{UCB}_{j,k}(t_j) \leq \text{UCB}_{j,k'}(t_j)$ must hold.
We have:
\begin{align}
\mathbb{E}\left[ L_{j,k,k'}(T_j) \right] 
&= \mathbb{E}\left[ \sum_{t_j=1}^{T_j} \mathbbm{1}\left\{ \text{UCB}_{j,k}(t_j) \leq \text{UCB}_{j,k'}(t_j) \land m_{h_j(t_j)}(j) = a_{k'} \right\} \right] \label{eq:expected_L} \\
&\leq 4 + 4\frac{\log T_j}{\Delta_{j,k,k'}^2}, \label{eq:ucb_bound_result}
\end{align}
where \eqref{eq:ucb_bound_result} follows from the standard UCB analysis~\cite{auer2002finite} (see Lemma~\ref{lem:standard_ucb_analysis} in Appendix~\ref{ap:pessimal} for the complete proof).

Now, we finalize the proof of Theorem~\ref{thm:pessimal_upper_bound}. Combining \eqref{eq:5.1}--\eqref{eq:ucb_bound_result}, we obtain:
\begin{align}
    \underline{R}_i(T_i) 
    &\leq \Delta_{i,\max} \cdot \sum_{j=1}^N \sum_{k=1}^K \sum_{k' \in \mathcal{W}'_{j,k}} \left( 4 + 4\frac{\log T_j}{\Delta_{j,k,k'}^2} \right) \nonumber \\
    &\leq \Delta_{i,\max} \cdot \sum_{j=1}^N \sum_{k=1}^K \left( 4K + \sum_{k' \in \mathcal{W}'_{j,k}} 4\frac{\log T_j}{\Delta_{j,k,k'}^2} \right).
\end{align}
Using that $\sum_{k' \in \mathcal{W}'_{j,k}} \Delta_{j,k,k'}^{-2} \leq \sum_{\ell=1}^{K-1}(\ell\Delta)^{-2} \leq 2\Delta^{-2}$ and assuming $T_i \approx T_j$ for all players $p_i, p_j \in \mathcal{P}$, we arrive at:
\begin{align}
    \underline{R}_i(T_i) 
    &\leq \Delta_{i,\max} NK \left( 4K + 8\log T_i/\Delta^2 \right)= \mathrm{O}\left( NK\log T_i/\Delta^2 \right),
\end{align} 
where $K = \mathrm{O}(\log T_i)$ is used in the last transformation. \qed
\end{proof}

While the previous method achieves a sublinear player-pessimal stable regret $\underline{R}_i(T_i)$, it cannot guarantee a sublinear upper bound for the player-optimal stable regret $\overline{R}_i(T_i)$ (\cite[Example 4]{liu2020competing} gives such an example). 
In what follows, we analyze the player-optimal stable regret $\overline{R}_i(T_i)$ for the AC-ETGS algorithm.

Before stating the theorem, we introduce several lemmas, whose proofs are provided in Appendix~\ref{ap:optimal}.
We define the failure event $F_i(t_i)$ for player $p_i$ at round $t_i$, representing that the empirical mean deviates significantly from the true mean for some player $p_j$ joining at the same global round $h_i(t_i)$:
\begin{equation}
    F_i(t_i) = \left \{ \exists (p_{j}, a_k) \in X_{h_i(t_i)} \text{ s.t. } |\hat{\mu}_{j,k}(t_{j}) - \mu_{j,k}| > \sqrt{\frac{\log t_{j}}{T_{j,k}(t_{j})}} \right \},
\end{equation}
where $X_{h_i(t_i)} \coloneqq \mathcal{P}_{h_i(t_i)} \times \mathcal{A}_{h_i(t_i)}$ is the set of pairs $(p_j, a_k)$ and $t_j \coloneqq h_j^{-1}(h_i(t_i))$.

\begin{lemma}\label{lem:badevent}
The expected number of rounds where $F_i(t_i)$ occurs is bounded by:
\[
\mathbb{E}\left[\sum_{t_i=1}^{T_i}\mathbbm{1}\{F_i(t_i)\}\right]\leq 4NK.
\]
\end{lemma}

\begin{lemma}
\label{lem:UCB,LCB}
Conditioned on the event $\neg F_i(t_i)$, for any player $p_{j} \in \mathcal{P}_{h_i(t_i)}$ at round $t_i$, if $\mathrm{UCB}_{j,k}(t_{j}) < \mathrm{LCB}_{j,k'}(t_{j})$, then the true means satisfy $\mu_{j,k} < \mu_{j,k'}$.
\end{lemma}

\begin{lemma}
\label{lem:times}
Consider player $p_i$ at round $t_i$. For any player $p_{j} \in \mathcal{P}_{h_i(t_i)}$, let $T_{j}(t_j) = \min_{k \in \mathcal{A}_{h_i(t_i)}} T_{j,k}(t_{j})$. Define the threshold $\bar{T}_{j} = 16\log T_{j} / \Delta^2$.
Conditioned on the event $\neg F_i(t_i)$, if $T_{j}(t_j) > \bar{T}_{j}$, then for any pair of arms $k, k'$ such that $\mu_{j,k} < \mu_{j,k'}$, the inequality $\mathrm{UCB}_{j,k}(t_{j}) < \mathrm{LCB}_{j,k'}(t_{j})$ holds.
\end{lemma}

\begin{theorem}\label{thm:optimal_upper_bound}
  Assume $T_i \approx T_j$ for all players $p_i, p_j \in \mathcal{P}$. 
When using Algorithm~\ref{alg:ac_etgs_algorithm}, the player-optimal stable regret for any player $p_i$ is bounded as
\begin{align*}
    \overline{R}_i(T_i)=\mathrm{O}(NK^2\log T_{i}/\Delta^2).
\end{align*}
\end{theorem}
\begin{proof}
We extend the analysis of ETGS algorithm~\cite{kong2023player} to the AC-ETGS algorithm.
Using the failure event $F_i(t_i)$, the player-optimal regret $\overline{R}_i(T_i)$ for player $p_i$ can be decomposed into two terms, one corresponding to rounds where the confidence bounds hold (main term) and one where they fail (failure term), as follows:
\begin{align}
    \overline{R}_i(T_i) 
    &= \sum_{t_i =1}^{T_i} \mathbb{E}\left[ \mu_{i, \overline{m}_{h_i(t_i)}(i)} - r_{i, m_{h_i(t_i)}(i)}(h_i(t_i)) \right] \nonumber \\
    &\leq \Delta_{i,\max} \cdot \mathbb{E}\left[ \sum_{t_i=1}^{T_i} \mathbbm{1}\{ m_{h_i(t_i)}(i) \neq \overline{m}_{h_i(t_i)}(i) \} \right] \nonumber \\
    &\leq \underbrace{\Delta_{i,\max} \cdot \mathbb{E}\left[ \sum_{t_i=1}^{T_i} \mathbbm{1}\{ m_{h_i(t_i)}(i) \neq \overline{m}_{h_i(t_i)}(i) , \neg F_i(t_i)\} \right]}_{ R_i^{\mathrm{main}}(T_i)} \nonumber\\
    & \quad + \underbrace{\Delta_{i,\max} \cdot \mathbb{E}\left[ \sum_{t_i=1}^{T_i} \mathbbm{1}\{ F_i(t_i) \} \right]}_{ R_i^{\mathrm{fail}}(T_i)}. \label{eq:decomp}
\end{align}
Lemma~\ref{lem:badevent} immediately gives a desired bound on the failure term $R_i^{\mathrm{fail}}(T_i)$, and in the following we concentrate the main term $R_i^{\mathrm{main}}(T_i)$.

Lemma~\ref{lem:UCB,LCB} ensures that conditioned on $\neg F_i(t_i)$, if Algorithm~\ref{alg:ac_etgs_algorithm} enters the exploitation rounds, the GS algorithm operates on a preference list that is consistent with the true preferences, yielding a player-optimal stable matching with zero regret. Therefore, the term $R_i^{\mathrm{main}}(T_i)$ only accumulates nonzero regret during the \textit{exploration} rounds, i.e., when the exploitation condition is not met.

Let $\mathcal{E}(t_i)$ denote the event that the exploitation condition of Algorithm~\ref{alg:ac_etgs_algorithm} is \emph{not} satisfied at round $t_i$. Based on Lemma \ref{lem:times}, under the event $\neg F_i(t_i)$, this $\mathcal{E}(t_i)$ implies that there exists at least one player-arm pair that has not been sampled sufficiently.
Thus, we can bound $R_i^{\mathrm{main}}(T_i)$ as follows:
\begin{align*}
& R_i^{\mathrm{main}}(T_i)\\
&=  \Delta_{i,\max} \cdot \mathbb{E}\left[\sum_{t_i=1}^{T_i} \mathbbm{1}\{ m_{h_i(t_i)}(i)\ne \overline{m}_{h_i(t_i)}(i), \neg F_i(t_i)\}\right] \\
&\le \Delta_{i,\max} \cdot \mathbb{E}\left[\sum_{t_i=1}^{T_i} \mathbbm{1}\{ \mathcal{E}(t_i), \neg F_i(t_i)\}\right] \\
&=  \Delta_{i,\max} \cdot \mathbb{E}\left[\sum_{t_i=1}^{T_i} \mathbbm{1}\left\{\mathcal{E}(t_i), \neg F_i(t_i), \exists (p_{j}, a_k) \in X_{h_i(t_i)} \text{ s.t. } T_{j,k}(t_{j}) \leq \frac{16\log T_{j}}{\Delta^2} \right\}\right]\\
&\leq \Delta_{i,\max} \sum_{p_{j} \in \mathcal{P}, a_k \in \mathcal{A}} \mathbb{E}\left[\sum_{t_i=1}^{T_i} \mathbbm{1}\left\{\mathcal{E}(t_i), \neg F_i(t_i),  (p_{j}, a_k) \in X_{h_i(t_i)}, T_{j,k}(t_{j}) \leq \frac{16\log T_{j}}{\Delta^2} \right\}\right].
\end{align*}

When the event $\mathcal{E}(t_i)$ occurs, the algorithm proceeds to the exploration rounds, where players are matched uniformly at random. Specifically, at round $t_i$, a player $p_{j}$ is matched with an arm $a_k$ with probability at least
\[\frac{1}{\max\{|\mathcal{P}_{h_i(t_i)}|,|\mathcal{A}_{h_i(t_i)}|\}} \geq \frac{1}{K}.\]
Therefore, in the exploration rounds, the expected number of rounds required to increment the counter $T_{j,k}(t_j)$ by 1 is at most $K$.
Since the total number of samples required for any pair $(p_j, a_k)$ is bounded by $16\log T_j / \Delta^2$, we obtain:
\begin{align*}
R_i^{\mathrm{main}}(T_i) &\leq \Delta_{i,\max} \cdot \sum_{p_{j} \in \mathcal{P}, a_k \in \mathcal{A}} \left( K\cdot 16\log T_j/\Delta^2 \right) \\
&\le \Delta_{i,\max} \cdot 16 NK^2 \log T_{i}/\Delta^2.
\end{align*}

Finally, combining the bounds for $R_i^{\mathrm{fail}}(T_i)$ and $R_i^{\mathrm{main}}(T_i)$, we obtain:
\begin{align}
\overline{R}_i(T_i) &\leq R_i^{\mathrm{main}}(T_i) + R_i^{\mathrm{fail}}(T_i) \nonumber \\
&\leq \Delta_{i,\max} \cdot 16 NK^2 \log T_{i}/\Delta^2 + 4NK \Delta_{i,\max} \nonumber \\
&= \mathrm{O} \left(NK^2\log T_{i}/\Delta^2\right). \nonumber
\end{align}
This concludes the proof. %
\qed
\end{proof}

In Appendix~\ref{ap:experiment}, we also present an empirical comparison between the theoretically guaranteed random matching and a heuristic approach using weighted matching during the exploration phase.
The results demonstrate that, in specific situations, the weighted matching approach may further improve the empirical performance compared to the theoretically guaranteed random matching.

\section{Conclusion}
In this paper, we have proposed \textit{Sleeping Competing Bandits}, incorporating the dynamic availability of arms and players into the standard competing bandits framework~\cite{liu2020competing}, and naturally extended the definitions of player-pessimal and player-optimal stable regret.
We have first demonstrated that without specific assumptions, the lower bound is nearly linear for both regret notions.
Consequently, under reasonable structural assumptions, we have proposed an algorithm achieving a player-pessimal stable regret upper bound of $\mathrm{O}(NK\log T_{i}/\Delta^2)$, which is asymptotically optimal.
Furthermore, we have designed a second algorithm achieving a player-optimal stable regret upper bound of $\mathrm{O}(NK^2\log T_{i}/\Delta^2)$.

Determining whether a tighter lower bound exists for player-optimal stable regret, or if an algorithm with a tighter upper bound can be constructed, remains an open question. Finally, while this work focused on stochastic reward distributions, extending the framework to adversarial settings presents an interesting avenue for future research.

\begin{credits}
\subsubsection{\ackname}
This work was supported by JSPS KAKENHI Grant Number JP25H01114 and JST CRONOS Japan Grant Number JPMJCS24K2.

\end{credits}

\bibliographystyle{splncs04} %
\bibliography{ref}           %

\clearpage

\appendix
\counterwithin{theorem}{section}
\counterwithin{lemma}{section}
\counterwithin{definition}{section}
\counterwithin{example}{section}
\counterwithin{equation}{section}
\counterwithin{algorithm}{section}

\chapter*{Supplementary Materials}

\section{Basics on Analysis of Stochastic Multi-Armed Bandit}
\subsection{Hoeffding's Inequality}
In the analysis of stochastic MAB problems, concentration inequalities are often employed to bound the deviation of sample means from true means. 
One such inequality is Hoeffding's inequality \cite{Hoeffding1963}, which provides a bound on the probability that the sum of bounded 
independent random variables deviates from its expected value.
\begin{theorem}[Hoeffding's Inequality \cite{Hoeffding1963}]\label{thm:hoeffding_inequality}
Let $Z_1, \dots, Z_n$ be independent bounded random variables with $Z_i \in [a, b]$ for all $i$, where $-\infty < a \le b < \infty$. Then
\[
\mathbb{P} \left( \frac{1}{n} \sum_{i=1}^n (Z_i - \mathbb{E}[Z_i]) \ge t \right) \le \exp \left( - \frac{2nt^2}{(b-a)^2} \right)
\]
and
\[
\mathbb{P} \left( \frac{1}{n} \sum_{i=1}^n (Z_i - \mathbb{E}[Z_i]) \le -t \right) \le \exp \left( - \frac{2nt^2}{(b-a)^2} \right)
\]
for all $t \ge 0$.
\end{theorem}

\subsection{Kullback--Leibler Divergence}\label{sec:KL_divergence}
In the analysis of stochastic MAB problems, the Kullback--Leibler (KL) divergence \cite{kullback1951information} is often used to measure the difference between two probability distributions.
\begin{definition}[Kullback--Leibler Divergence \cite{kullback1951information}]\label{def:kl_divergence}
  Let $P$, $Q$ be discrete probability distributions. The \emph{Kullback--Leibler (KL) divergence} from $P$ to $Q$ is defined as
  \[
  \mathrm{D}_{\mathrm{KL}}(P \| Q) = \sum_{x} P(x) \log \frac{P(x)}{Q(x)},
  \]
where $P(x)$ and $Q(x)$ are the probability mass functions of $P$ and $Q$, respectively.
\end{definition}
\begin{example}[Upper Bound on Bernoulli KL Divergence with Specific Means]\label{ex:bernoulli_kl_specific}
  We consider the KL divergence between two specific Bernoulli distributions: $P$ with mean $p = 0.5 - \delta$ and $Q$ with mean $q = 0.5 + \epsilon$, where $\delta, \epsilon > 0$.
  
  Using the inequality $\log y \le y - 1$, we derive an upper bound as follows:
  \begin{align}
  \mathrm{D}_{\mathrm{KL}}(p \| q) &= \sum_{x \in \{0, 1\}} P(x) \log \frac{P(x)}{Q(x)} \nonumber \\
  &\leq \sum_{x \in \{0, 1\}} P(x) \left( \frac{P(x)}{Q(x)} - 1 \right) \nonumber \\
  &= \sum_{x \in \{0, 1\}} \frac{P(x)^2}{Q(x)} - 1 \nonumber \\
  &= \frac{(0.5 - \delta)^2}{0.5 + \epsilon} + \frac{(0.5 + \delta)^2}{0.5 - \epsilon} - 1.
  \end{align}
  
  To simplify this expression, we use the algebraic identity $\sum \frac{P(x)^2}{Q(x)} - 1 = \sum \frac{(P(x)-Q(x))^2}{Q(x)}$.
  Noting that the difference in means is $|P(x) - Q(x)| = \delta + \epsilon$ for both $x=0$ and $x=1$, we have:
  \begin{align}
  \frac{(0.5 - \delta)^2}{0.5 + \epsilon} + \frac{(0.5 + \delta)^2}{0.5 - \epsilon} - 1 
  &= (\delta + \epsilon)^2 \left( \frac{1}{0.5 + \epsilon} + \frac{1}{0.5 - \epsilon} \right) \nonumber \\
  &= (\delta + \epsilon)^2 \left( \frac{(0.5 - \epsilon) + (0.5 + \epsilon)}{0.25 - \epsilon^2} \right) \nonumber \\
  &= \frac{(\delta + \epsilon)^2}{0.25 - \epsilon^2}.
  \end{align}
  This result provides a closed-form upper bound dependent on both $\delta$ and $\epsilon$.
\end{example}

\subsection{Finite-Time Instance-Dependent Lower Bound}\label{sec:change-of-measure}
Here, we present a key inequality used in deriving lower bounds for the expected regret in stochastic MAB problems.
\begin{theorem}[Lower Bound on the Number of Selections of a Suboptimal Arm ({\cite[Lemma 16.3]{lattimore2020bandit}})]\label{thm:lower_bound_suboptimal_arm}
  Let $\nu = (P_i)$ and $\nu' = (P'_i)$ be $K$-armed stochastic MAB instances that differ only in the distribution of the reward for arm 
  $a_i \in \mathcal{A}$. Assume that $a_i$ is suboptimal in $\nu$ and uniquely optimal in $\nu'$. Let $\lambda = \mu_i(\nu') - \mu_i(\nu)$. 
  Then, for any policy $\pi$,
\[
    \mathbb{E}_{\nu, \pi}[T_i(T)] \geq \frac{\log \left( \frac{\min\{\lambda - \Delta_i(\nu), \Delta_i(\nu)\}}{4} \right) + \log(T) - \log(R_{\nu}(\pi, T) + R_{\nu'}(\pi, T))}{\mathrm{D}_{\mathrm{KL}}(P_i \| P'_i)},
\]
where $\mathbb{E}_{\nu, \pi}[T_i(T)]$ is the expected number of selections of a suboptimal arm $a_i$ in $\nu$ if following policy $\pi$, and $\Delta_i(\nu) = \mu_{i^\star}(\nu) - \mu_i(\nu)$ is the gap between the optimal arm and 
the suboptimal arm $a_i$ in $\nu$. 
\end{theorem}

This theorem establishes a lower bound on the number of times the arm $a_i$, which is suboptimal in $\nu$, must be selected to achieve sublinear regret in both instances $\nu$ and $\nu'$.
Theorem \ref{thm:lower_bound_suboptimal_arm} is derived by combining the Bretagnolle--Huber inequality (\cite{bretagnolle2006estimation} and \cite[Theorem 14.2]{lattimore2020bandit}) and the divergence decomposition lemma \cite[Lemma 15.1]{lattimore2020bandit}.

\section{Gale--Shapley Algorithm for Stable Matching}\label{sec:GS_Algorithm}
The following algorithm is known to compute the player-optimal stable matching.

\begin{algorithm}[H]
\caption{Gale--Shapley Algorithm (Player-Proposing) \cite{gale1962college}}
\label{alg:gs_algorithm}
\textbf{Input:} An instance $I = (\mathcal{P},  \mathcal{A}, \{\succ_p\}_{p \in \mathcal{P}}, \{(\succ_a, c_a)\}_{a \in \mathcal{A}})$. \\
\textbf{Output:} The player-optimal stable matching $\overline{m}$.

\begin{enumerate}
    \item Initialize set of free players $P=\mathcal{P}$, and for all arms $a \in \mathcal{A}$, let $M_a = \emptyset$ (set of matched players).
    \item While $P$ is not empty, repeat the following:
    \begin{enumerate}
        \item Choose $p\in P$. If $p$ has already proposed to every arm, remove $p$ from $P$. Otherwise do the following procedures.
        \begin{enumerate}
             \item Let $a^\star$ be the highest-ranked arm in $p$'s preference list to whom $p$ has not yet proposed.\\
             $p$ proposes to $a^\star$.
             \item If $|M_{a^\star}| < c_{a^\star}$ (arm $a^\star$ has a vacancy):
             \begin{itemize}
                 \item Add $p$ to $M_{a^\star}$ and remove $p$ from $P$.
             \end{itemize}
             \item If $|M_{a^\star}| = c_{a^\star}$ (arm $a^\star$ is full):
             \begin{itemize}
                 \item Let $p_{\text{worst}}$ be the least preferred player in $M_{a^\star}$ according to $a^\star$'s preference.
                 \item If $a^\star$ prefers $p$ to $p_{\text{worst}}$:
                 \begin{itemize}
                     \item Remove $p_{\text{worst}}$ from $M_{a^\star}$ and add $p_{\text{worst}}$ to $P$.
                     \item Add $p$ to $M_{a^\star}$ and remove $p$ from $P$.
                 \end{itemize}
             \end{itemize}
        \end{enumerate}
    \end{enumerate}
    \item Return the set of matched pairs defined by $\{M_a\}_{a \in \mathcal{A}}$ as $\overline{m}$.
\end{enumerate}
\end{algorithm}

The player-pessimal stable matching is computed analogously by swapping the roles of players and arms, where the capacity of each player is regarded as one and each arm proposes as long as it has a vacant seat and a player remains in its preference list.

\section{Regret Lower Bounds}
\subsection{Regret Lower Bound without Assumptions}\label{ap:general_lowerbound}
Here we give a complete proof of Theorem~\ref{thm:generallowerbound} presented in Section~\ref{sec:lower}.

\restate{Theorem}{thm:generallowerbound}
 For any policy and any constant $c \in (0, 1)$, there exists a problem instance (a collection of reward distributions) such that the player-optimal stable regret $\overline{R}_i(T_i)$ and player-pessimal stable regret $\underline{R}_i(T_i)$ for some player $p_i$ satisfies:
 \[
     \overline{R}_i(T_i)\geq \underline{R}_i(T_i) = \omega(T_i^c).
 \]
\normalfont
\begin{proof}
To derive the lower bound stated in Theorem~\ref{thm:generallowerbound}, we utilize a change-of-measure argument. We construct a reference instance, \textbf{Instance 1}, and a family of alternative instances, $\{\textbf{Instance }k\}_{k}$, one for each competing player $p_k$. 

\textbf{Common Settings.}
We consider a subset of players $\mathcal{P}_1 \subseteq \mathcal{P}$ and a subset of arms, denoted as $\mathcal{A}_1 \subseteq \mathcal{A}$.
We focus on a target player $p_1 \in \mathcal{P}_1$, and rounds that $p_1$ is available, i.e., $H_1\subseteq \mathcal{T}$.
Let $\mathcal{A}_1 = \{a_1, a_2\}$ be available at all rounds in $H_1$, and the capacities of both arms are fixed to $c_1(t_1) = c_2(t_1) = 1$ for all $t_1 \in H_1$.
The availability of the other players $p_k \in \mathcal{P}_1 \setminus \{p_1\}$ is defined to partition the time horizon $T_1$.
Let $L=\mathrm{\omega}(1)$ (w.r.t.~$T_1$) be a block length. For each $p_k \in \mathcal{P}_1 \setminus \{p_1\}$, the set of available rounds $H_k$ is defined as:
\[
H_k = \{ t_1 \in H_1 \mid L \cdot (k-2) < t_1 \le L \cdot (k-1) \}.
\]
In any round $t_1$, exactly one competing player $p_k$ is available alongside $p_1$.

For all rounds $t_1 \in H_1$, both arms $a_1$ and $a_2$ prefer the competing player $p_k$ ($k \ge 2$) over the target player $p_1$.
That is, in the preference ranking of arm $a_j$, we have $p_k \succ_{j} p_1$. This puts $p_1$ at a disadvantage if they compete for the same arm.

\textbf{Unknown Player Preferences (Instance Dependent).}
Let $\delta \in (0, 0.25)$ be a gap parameter. We define the instances based on the mean rewards $\{\mu_{i,j}\}$.

\begin{itemize}
    \item \textbf{Instance 1 (Reference Instance):}
    In this instance, the target player $p_1$ prefers $a_1$, while all competing players prefer $a_2$.
    \begin{itemize}
        \item Player $p_1$: $\mu_{1,1} = 0.5 + \delta$, $\mu_{1,2} = 0.5$.
        \item Competing players $p_j$ ($j \ge 2$): $\mu_{j,1} = 0.5$, $\mu_{j,2} = 0.5 + \delta$.
    \end{itemize}
    The unique stable matching is always $\{(p_1, a_1), (p_j, a_2)\}$.
    
    \item \textbf{Instance $k$ (Alternative Instance for $p_k$):}
    For a specific competing player $p_k$ ($k \ge 2$), we define an alternative instance where only $p_k$'s preference is flipped to favor $a_1$.
    \begin{itemize}
        \item Player $p_k$: $\mu_{k,1} = 0.5 + 2\delta$, $\mu_{k,2} = 0.5 + \delta$. %
        \item All other players $p_j$ ($j \neq k$) have the same rewards as in Instance 1.
    \end{itemize}
    Whenever $j \neq k$, the unique stable matching is $\{(p_1, a_1), (p_j, a_2)\}$, but when $j = k$, it is $\{(p_1, a_2), (p_j, a_1)\}$.
\end{itemize}
\textbf{Regret Analysis}
We analyze the regret of player $p_1$ in Instance 1.
Ideally, in Instance 1, any competing player $p_j$ should select $a_2$.
However, to distinguish Instance 1 from Instance $k$, player $p_k$ must explore $a_1$.

Let $P_{k,1}(1)$ and $P_{k,1}(k)$ be the reward distributions observed by player $p_k$ (when pulling arm $a_1$) in Instance 1 and Instance $k$, respectively.
The key observation will be shown as Lemma~\ref{lem:selection of suboptimal arm}. It implies that if a policy is consistent (i.e., achieves sublinear regret) on both Instance 1 and Instance $k$,
player $p_k$ must select $a_1$ a sufficient number of times in Instance 1 because the distributions of arm $a_1$ for $p_k$ are close in KL-divergence between the two instances.

Whenever $p_k$ selects $a_1$ in Instance 1, $p_1$ is blocked and incurs a regret of $\delta$.
Thus, when $\underline{R}_k(L) = \mathrm{O}(L^c)$ for all $k \ge 2$, the total regret of $p_1$ is lower bounded by summing the expected number of suboptimal selections by each $p_k$:
\begin{align}
\underline{R}_1(T_1) &= \sum_{t_1=1}^{T_1} \left( \mu_{1, \underline{m}_{t_1}(1)} - \mathbb{E}[r_{1, m_{t_1}(1)}(t_1)] \right)  \\
&\geq  \delta \sum_{k=2}^{\lfloor T_1 / L \rfloor} \mathbb{E} [T_{k, 1}(L)] \label{eq:4.2} \\
&\geq \delta \sum_{k=2}^{\lfloor T_1 / L \rfloor} \frac{(1-c)\log L-\log(2C) + \log(\delta/4)}{D_{\mathrm{KL}}(P_{k, 1}(1) \| P_{k, 1}(k))} \label{eq:4.3}\\
&=\mathrm{\Omega}((T_1/L)\cdot \log L)=\mathrm{\Omega}(T_1^c \log T_1)=\mathrm{\omega}(T_1^c), \label{eq:4.4}
\end{align}
where $T_{k, 1}(L) = \sum_{t_1 \in H_k} \mathbbm{1}\{m_{t_1}(k) = a_1\}$ is the number of times $p_k$ selects $a_1$, \eqref{eq:4.3} follows from Lemma \ref{lem:selection of suboptimal arm} below ($C > 0$ is a hidden constant in the bound $\underline{R}_k(L) = \mathrm{O}(L^c)$), and \eqref{eq:4.4} follows by taking $L=\mathrm {\Theta} (T_1^{1-c})=\mathrm{\omega}(1)$.
\qed
\end{proof}

\begin{lemma}\label{lem:selection of suboptimal arm}
  For any policy that achieves $\underline{R}_k(L) \le C L^c$ for some constant $C>0$ in both Instance 1 and Instance $k$, the expected number of suboptimal arm selections by $p_k$ in Instance 1 satisfies:
  \[
  \mathbb{E} [ T_{k, 1}(L) ] \geq \frac{(1-c)\log L-\log(2C) + \log(\delta/4)}{D_{\mathrm{KL}}(P_{k, 1}(1) \| P_{k, 1}(k))}.
  \]
\end{lemma}
\begin{proof}
Player $p_k$'s problem is equivalent to a standard stochastic MAB problem with two arms.
We apply Theorem \ref{thm:lower_bound_suboptimal_arm} with $\nu = \text{Instance 1}$ and $\nu' = \text{Instance $k$}$.
\begin{itemize}
    \item Suboptimal arm in $\nu$: $a_1$ (gap $\delta_i(\nu) = \delta$).
    \item Optimal arm in $\nu'$: $a_1$ (mean increases by $\lambda = (0.5+2\delta) - 0.5 = 2\delta$).
\end{itemize}
The numerator term in Theorem \ref{thm:lower_bound_suboptimal_arm} becomes:
\[
\min\{\lambda - \delta_i(\nu), \delta_i(\nu)\} = \min\{2\delta - \delta, \delta\} = \delta.
\]

Using the assumption $R_{\nu} + R_{\nu'} \leq 2CL^c$ (roughly), and substituting into the theorem:
\begin{align}
  \mathbb{E}[T_{k, 1}(L)] & \geq \frac{\log \left( \frac{\delta}{4} \right) + \log L - \log (2CL^c)}{D_{\mathrm{KL}}(P_{k, 1}(1) \| P_{k, 1}(k))} \\
  & = \frac{(1-c)\log L-\log(2C) + \log(\delta/4)}{D_{\mathrm{KL}}(P_{k, 1}(1) \| P_{k, 1}(k))}. \qquad \qed
\end{align}
\end{proof}

\subsection{Regret Lower Bound under Assumptions}\label{ap:lowerbound}
Here we complete the proof of Theorem~\ref{thm:lowerbound} presented in Section~\ref{sec:lower}.

\restate{Theorem}{thm:lowerbound}
Assume $T_i \approx T_j$ for all players $p_i, p_j \in \mathcal{P}$ and $K=\mathrm{O}(\log T_i)$. 
Then, for any $\alpha \in [0, 1)$ and any $\alpha$-consistent policy, there exists a problem instance such that the player-optimal
stable regret $\overline{R}_i(T_i)$ and player-pessimal stable regret $\underline{R}_i(T_i)$ for some player $p_i$ satisfy:
\[
\overline{R}_i(T_i)\geq \underline{R}_i(T_i) = \mathrm{\Omega}\left( N(K-N+1)\log T_{i}/\Delta^2 \right).
\]
\normalfont

First, we provide a lemma that justifies the transition from \eqref{eq:4.8} to \eqref{eq:4.9} in the proof of Theorem~\ref{thm:lowerbound} in the main text.
Recall that we set $c = (1 + \alpha) / 2 \in (\alpha, 1)$ and $\underline{R}_k(|\mathcal{T}_k|) = \mathrm{O}(|\mathcal{T}_k|^c)$ as the policy is $\alpha$-consistent.
\begin{lemma}\label{lem:selection of suboptimal arm 2}
  For any policy that achieves $\underline{R}_k(|\mathcal{T}_k|) \le C |\mathcal{T}_k|^c$ for some constant $C>0$ in both Instance 1 and Instance $(i,k)$, the expected number of suboptimal arm selections by $p_i$ in Instance 1 satisfies:
  \[
  \mathbb{E} [ T_{i, k}(|\mathcal{T}_k|) ] \geq \frac{(1-c)\log |\mathcal{T}_k|-\log(2C) + \log(\epsilon/4)}{D_{\mathrm{KL}}(P_{i,k}(1) \| P_{i,k}(i,k))}. 
  \]
\end{lemma}
\begin{proof}
Player $p_i$'s problem is equivalent to a standard stochastic MAB problem with two arms.
We apply Theorem \ref{thm:lower_bound_suboptimal_arm} with $\nu = \text{Instance 1}$ and $\nu' = \text{Instance $(i,k)$}$.
\begin{itemize}
    \item Suboptimal arm in $\nu$: $a_k$ (gap: $\Delta(\nu) = \Delta$).
    \item Optimal arm in $\nu'$: $a_k$ (mean increases by: $\lambda = \Delta+\epsilon$).
\end{itemize}
The numerator term in Theorem \ref{thm:lower_bound_suboptimal_arm} becomes:
\[
\min\{\lambda - \Delta(\nu), \Delta(\nu)\} \geq \min \left \{ \epsilon ,  \Delta \right \} =\epsilon.
\]

Using the assumption $R_{\nu} + R_{\nu'} \leq 2C|\mathcal{T}_k|^c$ (roughly), and substituting into the theorem:
\begin{align*}
  \mathbb{E}[T_{i, k}(|\mathcal{T}_k|)] & \geq \frac{\log \left( \frac{\epsilon}{4} \right) + \log (|\mathcal{T}_k|) - \log (2C|\mathcal{T}_k|^c)}{D_{\mathrm{KL}}(P_{i,k}(1) \| P_{i,k}(i,k))} \\
  & = \frac{(1-c)\log |\mathcal{T}_k|-\log(2C) + \log(\epsilon/4)}{D_{\mathrm{KL}}(P_{i,k}(1) \| P_{i,k}(i,k))}. \qquad \qed
\end{align*}
\end{proof}

Finally, we provide the upper bound on KL divergence that justifies the transition from \eqref{eq:4.9} to \eqref{eq:4.10} in the proof of Theorem~\ref{thm:lowerbound} in the main text.
We utilize the specific parameter choices of our hard instance: $\Delta$ is small enough compared to $0.1$, and $\epsilon$ is chosen such that $\epsilon \le 0.1 \Delta$. 
Under these conditions, the KL divergence term in the denominator is bounded as follows (see Example~\ref{ex:bernoulli_kl_specific}):
\begin{align*} 
&D_{\mathrm{KL}}(P_{i,k}(1) \| P_{i,k}(i,k)) \\
    &\le \frac{(\Delta + \epsilon)^2}{0.25 - \epsilon^2} 
    \le \frac{(\Delta + 0.1\Delta)^2}{0.25 - (0.1\Delta)^2} 
    = \frac{1.21 \Delta^2}{0.25 - 0.01\Delta^2} 
    \le \frac{1.21 \Delta^2}{0.249} \le 5 \Delta^2.
\end{align*}

\section{Regret Upper Bounds}
\subsection{Player-Pessimal Stable Regret Upper Bound}\label{ap:pessimal}
Here we complete the proof of Theorem~\ref{thm:pessimal_upper_bound} presented in Section~\ref{sec:upper}.

\restate{Theorem}{thm:pessimal_upper_bound}
Assume $T_i \approx T_j$ for all players $p_i, p_j \in \mathcal{P}$ and $K=\mathrm{O}(\log T_i)$. 
When using Algorithm~\ref{alg:ac_ucb_algorithm}, the player-pessimal stable regret for any player $p_i$ is bounded as
\begin{align*}
    \underline{R}_i(T_i) =\mathrm{O}(NK\log T_{i}/\Delta^2).
\end{align*}
\par\vspace{\topsep}
\normalfont

First, we provide a lemma that justifies the transition from \eqref{eq:5.4} to \eqref{eq:5.5} in the proof of Theorem~\ref{thm:pessimal_upper_bound} in the main text.

\begin{lemma}\label{lem:output of GS}
    For any round $t$ and matching $m_t$ produced by the AC-UCB algorithm (Algorithm~\ref{alg:ac_ucb_algorithm}), 
    there does not exist a blocking triplet of the form $(p_j, a_k, \emptyset)$ for $m_t$.
\end{lemma}
\begin{proof}
    Assume, for the sake of contradiction, that there exists a blocking triplet $(p_j, a_k, \emptyset)$ for the matching $m_t$.
    By definition of a blocking triplet, this implies that player $p_j$ is unmatched ($m_t(j) = 0$) but prefers arm $a_k$ ($\mu_{j,k} > 0$).
    Furthermore, arm $a_k$ is either not full or prefers $p_j$ over one of its current partners.
    
    The AC-UCB algorithm employs the Gale--Shapley (GS) algorithm using the UCB indices as preferences. 
    Since the GS algorithm guarantees stability with respect to the input preferences, and the algorithm submits indices for all arms (implying acceptability), such a blocking triplet cannot occur. 
    Specifically, if $p_j$ preferred $a_k$ based on its index, $p_j$ would have proposed to $a_k$. 
    Given that $a_k$ is available (or prefers $p_j$), the mechanism ensures they would be matched.
    Thus, $p_j$ cannot remain unmatched while such an arm $a_k$ exists, contradicting the assumption.
\qed
\end{proof}

Finally, we provide a lemma on the UCB analysis that justifies the transition from \eqref{eq:expected_L} to \eqref{eq:ucb_bound_result} in the proof of Theorem~\ref{thm:pessimal_upper_bound} in the main text.
\begin{lemma}
\label{lem:standard_ucb_analysis}
    The term in \eqref{eq:expected_L} is upper-bounded as:
    \[
    \mathbb{E}\left[ \sum_{t_j=1}^{T_j} \mathbbm{1}\left\{ \mathrm{UCB}_{j,k}(t_j) \leq \mathrm{UCB}_{j,k'}(t_j) \land m_{h_j(t_j)}(j) = a_{k'} \right\} \right] \leq 4 + \frac{4\log T_j}{\Delta_{j,k,k'}^2}.
    \]
\end{lemma}
\begin{proof}
    Let $T_{j,k}(t_j)$ denote the number of times arm $k$ has been selected by player $j$ up to round $t_j$.
    We define the following events regarding the estimation of arms:
    \begin{itemize}
        \item Event $A_{j,k,k'}(t_j)$ (Underestimation of optimal arm $a_k$):
        \[
        A_{j,k,k'}(t_j) = \left\{ \hat{\mu}_{j,k} \leq \mu_{j,k} - c_{t_j, T_{j,k}(t_j)} \land m_{h_j(t_j)}(j) = a_{k'} \right\}.
        \]
        \item Event $B_{j,k,k'}(t_j)$ (Overestimation of suboptimal arm $a_{k'}$):
        \[
        B_{j,k,k'}(t_j) = \left\{ \hat{\mu}_{j,k'} \geq \mu_{j,k'} + c_{t_j, T_{j,k'}(t_j)} \land m_{h_j(t_j)}(j) = a_{k'} \right\}.
        \]
        \item Event $C_{j,k,k'}(t_j)$ (Small gap implies insufficient samples):
        \[
        C_{j,k,k'}(t_j) = \left\{ \mu_{j,k} < \mu_{j,k'} + 2c_{t_j, T_{j,k'}(t_j)} \land m_{h_j(t_j)}(j) = a_{k'} \right\},
        \]
    \end{itemize}
    where $c_{t,s} = \sqrt{\frac{ \log t}{s}}$ is the confidence radius.
    The inequality $\text{UCB}_{j,k}(t_j) \leq \text{UCB}_{j,k'}(t_j)$ implies that at least one of the events of these three types must occur.
    Thus, we can decompose the expectation as:
    \begin{align}
        &\mathbb{E}\left[ \sum_{t_j=1}^{T_j} \mathbbm{1}\left\{ \text{UCB}_{j,k}(t_j) \leq \text{UCB}_{j,k'}(t_j) \land m_{h_j(t_j)}(j) = a_{k'} \right\} \right] \nonumber \\
        &\leq \underbrace{\mathbb{E}\left[ \sum_{t_j=1}^{T_j} \mathbbm{1}\{ A_{j,k,k'}(t_j) \} \right]}_{(a)} 
        + \underbrace{\mathbb{E}\left[ \sum_{t_j=1}^{T_j} \mathbbm{1}\{ B_{j,k,k'}(t_j) \} \right]}_{(b)} 
        + \underbrace{\mathbb{E}\left[ \sum_{t_j=1}^{T_j} \mathbbm{1}\{ C_{j,k,k'}(t_j) \} \right]}_{(c)}.
    \end{align}
    
    First, we bound term $(a)$ as follows:
 \begin{align}
    (a) &\le \mathbb{E}\left[ \sum_{t_j=1}^{T_j} \mathbbm{1}\left\{ \hat{\mu}_{j,k} \leq \mu_{j,k} - c_{t_j, T_{j,k}(t_j)} \right\} \right] \nonumber \\
    &= \sum_{t_j=1}^{T_j} \sum_{\tau=1}^{t_j} \mathbb{P} \left( T_{j,k}(t_j)=\tau \land \hat{\mu}_{j,k} \leq \mu_{j,k} - \sqrt{\frac{\log t_j}{\tau}} \right) \nonumber \\
    &= \sum_{t_j=1}^{T_j} \sum_{\tau=1}^{t_j} \mathbb{P}(T_{j,k}(t_j)=\tau) \cdot \mathbb{P} \left( \hat{\mu}_{j,k} \leq \mu_{j,k} - \sqrt{\frac{\log t_j}{\tau}} \;\middle|\; T_{j,k}(t_j)=\tau \right) \nonumber \\
    &\leq \sum_{t_j=1}^{T_j} \sum_{\tau=1}^{t_j} \mathbb{P}(T_{j,k}(t_j)=\tau) \cdot \exp(-2 \log t_j) \label{eq:hoeffding_tight} \\
    &= \sum_{t_j=1}^{T_j} t_j^{-2} \underbrace{\sum_{\tau=1}^{t_j} \mathbb{P}(T_{j,k}(t_j)=\tau)}_{{} = 1} \nonumber \\
    &\le \sum_{t_{j}=1}^{\infty} t_j^{-2} \le 2, \label{eq:bound_a}
\end{align}
    where \eqref{eq:hoeffding_tight} is derived from the Hoeffding inequality (Theorem~\ref{thm:hoeffding_inequality}).
    Similarly, %
    \begin{align}
        (b) \leq 2. \label{eq:bound_b}
    \end{align}
    
    Finally, we bound term $(c)$. The event $C_{j,k,k'}(t_j)$ implies
    \begin{align*}
        \Delta_{j,k,k'} < 2\sqrt{\frac{\log t_j}{T_{j,k'}(t_j)}},
    \end{align*}
    which is rephrased as
    \begin{align*}
        T_{j,k'}(t_j) < \frac{4 \log t_j}{\Delta_{j,k,k'}^2}.
    \end{align*}
    Since $T_{j,k'}(t_j)$ increments each time the event occurs (as $m(j)=k'$), the total count is bounded by:
    \begin{align}
        (c) \leq \frac{4 \log T_j}{\Delta_{j,k,k'}^2}. \label{eq:bound_c}
    \end{align}
    Combining \eqref{eq:bound_a}, \eqref{eq:bound_b}, and \eqref{eq:bound_c} completes the proof. \qed
\end{proof}

\subsection{Player-Optimal Stable Regret Upper Bound}\label{ap:optimal}
Here we complete the proof of Theorem~\ref{thm:optimal_upper_bound} presented in Section~\ref{sec:upper}.

\restate{Theorem}{thm:optimal_upper_bound}
Assume $T_i \approx T_j$ for all players $p_i, p_j \in \mathcal{P}$. 
When using Algorithm~\ref{alg:ac_etgs_algorithm}, the player-optimal stable regret for any player $p_i$ is bounded as
\begin{align*}
    \overline{R}_i(T_i)=\mathrm{O}(NK^2\log T_{i}/\Delta^2).
\end{align*}
\par\vspace{\topsep}
\normalfont

We provide proofs of lemmas used in the proof of Theorem~\ref{thm:optimal_upper_bound} in main text. 
Recall that we define the failure event $F_i(t_i)$ for player $p_i$ at round $t_i$ as follows:
\begin{equation*}
    F_i(t_i) = \left \{ \exists (p_{j}, a_k) \in X_{h_i(t_i)} \text{ s.t. } |\hat{\mu}_{j,k}(t_{j}) - \mu_{j,k}| > \sqrt{\frac{\log t_{j}}{T_{j,k}(t_{j})}} \right \},
\end{equation*}
where $X_{h_i(t_i)} = \mathcal{P}_{h_i(t_i)} \times \mathcal{A}_{h_i(t_i)}$ is the set of pairs $(p_j, a_k)$ and $t_j = h_j^{-1}(h_i(t_i))$.

\restate{Lemma}{lem:badevent}
The expected number of rounds where $F_i(t_i)$ occurs is bounded by:
\[
\mathbb{E}\left[\sum_{t_i=1}^{T_i}\mathbbm{1}\{F_i(t_i)\}\right]\leq 4NK.
\]
\par\vspace{\topsep}
\normalfont
\begin{proof}
By applying the union bound over all players $p_j \in \mathcal{P}$ and arms $a_k \in \mathcal{A}$, we have:
\begin{align}
&\mathbb{E}\left[\sum_{t_i=1}^{T_i} \mathbbm{1}\{F_i(t_i)\}\right]\nonumber\\
&\leq \sum_{p_{j} \in \mathcal{P}, a_k \in \mathcal{A}}  \mathbb{E}\left[\sum_{t_i=1}^{T_i} \mathbbm{1}\left\{ |\hat{\mu}_{j,k}(t_{j}) - \mu_{j,k}| > \sqrt{\frac{ \log t_{j}}{ T_{j,k}(t_{j})}} \land (p_{j} ,a_{k})\in X_{h_i(t_i)}\right\}\right] \nonumber \\
&\leq \sum_{p_{j} \in \mathcal{P}, a_k \in \mathcal{A}}  \mathbb{E}\left[\sum_{t_{j}=1}^{T_{j}} \mathbbm{1}\left\{ |\hat{\mu}_{j,k}(t_{j}) - \mu_{j,k}| > \sqrt{\frac{ \log t_{j}}{ T_{j,k}(t_{j})}}\right\}\right]. \label{eq:D6}
\end{align}
We can rewrite the expectation by summing over the possible values of the counter $T_{j,k}(t_j)$:
\begin{align}
&\mathbb{E}\left[\sum_{t_{j}=1}^{T_{j}} \mathbbm{1}\left\{ |\hat{\mu}_{j,k}(t_{j}) - \mu_{j,k}| > \sqrt{\frac{ \log t_{j}}{ T_{j,k}(t_{j})}}\right\}\right] \nonumber \\
&= \sum_{t_{j}=1}^{T_{j}} \sum_{\tau=1}^{t_{j}} \mathbb{P}\left(T_{j,k}(t_{j})=\tau \land |\hat{\mu}_{j,k}(t_{j}) - \mu_{j,k}| > \sqrt{\frac{ \log t_{j}}{ T_{j,k}(t_{j})}}\right) \nonumber \\
&= \sum_{t_{j}=1}^{T_{j}} \sum_{\tau=1}^{t_{j}} \mathbb{P}(T_{j,k}(t_{j})=\tau) \cdot \mathbb{P}\left(|\hat{\mu}_{j,k}(t_{j}) - \mu_{j,k}| > \sqrt{\frac{ \log t_{j}}{\tau}} \;\middle|\; T_{j,k}(t_{j})=\tau\right).
\end{align}
Applying Hoeffding's inequality as in the proof of Lemma~\ref{lem:standard_ucb_analysis}, we obtain:
\begin{align}
\text{RHS of \eqref{eq:D6}} &\leq \sum_{p_{j} \in \mathcal{P}, a_k \in \mathcal{A}} \sum_{t_{j}=1}^{T_{j}} \sum_{\tau=1}^{t_{j}} \mathbb{P}(T_{j,k}(t_{j})=\tau) \cdot 2\exp(- 2\log t_{j}) \nonumber \\
&= \sum_{p_{j} \in \mathcal{P}, a_k \in \mathcal{A}} \sum_{t_{j}=1}^{T_{j}} 2t_j^{-2} \cdot \underbrace{\sum_{\tau=1}^{t_{j}} \mathbb{P}(T_{j,k}(t_{j})=\tau)}_{=1} \nonumber \\
&\leq \sum_{p_{j} \in \mathcal{P}, a_k \in \mathcal{A}} 2 \sum_{t_{j}=1}^{\infty} t_j^{-2} \nonumber \\
&\leq \sum_{p_{j} \in \mathcal{P}, a_k \in \mathcal{A}} 4 = 4NK. \nonumber \qquad \qed
\end{align}
\end{proof}

\restate{Lemma}{lem:UCB,LCB}
Conditioned on the event $\neg F_i(t_i)$, for any player $p_{j} \in \mathcal{P}_{h_i(t_i)}$ at round $t_i$, if $\mathrm{UCB}_{j,k}(t_{j}) < \mathrm{LCB}_{j,k'}(t_{j})$, then the true means satisfy $\mu_{j,k} < \mu_{j,k'}$.
\par\vspace{\topsep}
\normalfont
\begin{proof}
Under the event $\neg F_i(t_i)$, the true mean lies within the confidence interval defined by LCB and UCB. Specifically, for any relevant arm $k$:
\[
\mathrm{LCB}_{j,k}(t_{j}) = \hat{\mu}_{j,k}(t_{j}) - \sqrt{\frac{\log t_{j}}{T_{j,k}(t_{j})}} \le \mu_{j,k} \le \hat{\mu}_{j,k}(t_{j}) + \sqrt{\frac{\log t_{j}}{T_{j,k}(t_{j})}} = \mathrm{UCB}_{j,k}(t_{j}).
\]
Therefore, the condition $\mathrm{UCB}_{j,k}(t_{j}) < \mathrm{LCB}_{j,k'}(t_{j})$ implies the following chain of inequalities:
\[
\mu_{j,k} \le \mathrm{UCB}_{j,k}(t_{j}) < \mathrm{LCB}_{j,k'}(t_{j}) \le \mu_{j,k'}.
\]
Thus, we conclude that $\mu_{j,k} < \mu_{j,k'}$.
\qed
\end{proof}

\restate{Lemma}{lem:times}
Consider player $p_i$ at round $t_i$. For any player $p_{j} \in \mathcal{P}_{h_i(t_i)}$, let $T_{j}(t_j) = \min_{k \in \mathcal{A}_{h_i(t_i)}} T_{j,k}(t_{j})$. Define the threshold $\bar{T}_{j} = 16\log T_{j} / \Delta^2$.
Conditioned on the event $\neg F_i(t_i)$, if $T_{j}(t_j) > \bar{T}_{j}$, then for any pair of arms $k, k'$ such that $\mu_{j,k} < \mu_{j,k'}$, the inequality $\mathrm{UCB}_{j,k}(t_{j}) < \mathrm{LCB}_{j,k'}(t_{j})$ holds.
\par\vspace{\topsep}
\normalfont 

\begin{proof}
We proceed by contradiction. Assume there exist arms $k, k'$ such that $\mu_{j,k} < \mu_{j,k'}$ but the confidence intervals overlap or are inverted, i.e., $\mathrm{UCB}_{j,k}(t_{j}) \ge \mathrm{LCB}_{j,k'}(t_{j})$.
Given $\neg F_i(t_i)$, the true means are contained within the confidence bounds:
\[
\mu_{j,k'} - 2\sqrt{\frac{\log t_{j}}{T_{j}(t_{j})}} \le \mathrm{LCB}_{j,k'}(t_{j}) \le \mathrm{UCB}_{j,k}(t_{j}) \le \mu_{j,k} + 2\sqrt{\frac{\log t_{j}}{T_{j}(t_{j})}}.
\]
This inequality implies that the gap between the means satisfies:
\[
\Delta_{j,k,k'} = \mu_{j,k'} - \mu_{j,k} \le 4\sqrt{\frac{\log t_{j}}{T_{j}(t_{j})}}.
\]
Rearranging for $T_{j}(t_{j})$, we get:
\[
T_{j}(t_{j}) \le \frac{16\log t_{j}}{\Delta_{j,k,k'}^2} \le \frac{16\log T_{j}}{\Delta^2} = \bar{T}_j.
\]
This contradicts the hypothesis that $T_{j}(t_{i}) > \bar{T}_{j}$. Thus, the confidence intervals must be separated correctly.
\qed
\end{proof}

\section{Experimental Details}\label{ap:experiment}
\subsection{Weighted Exploration}
In our experiments, we compared our proposed AC-ETGS algorithm with a variant that employs maximum-weight matchings instead of uniform random matchings during the exploration rounds. 
This approach actively prioritizes player-arm pairs that have been under-explored, which is intuitively expected to improve the exploration efficiency.
Specifically, let $N_{i,j}(t)$ denote the number of times player $i$ has successfully pulled arm $j$ up to round $t\in \mathcal{T}$. We define the matching weight for this pair as $w_{i,j}(t) = \frac{1}{N_{i,j}(t) + 1}$. During each exploration round, the central platform computes a matching that maximizes the sum of weights among all currently available players and arms. This maximum weight bipartite matching problem is efficiently solved using the Hungarian algorithm~\cite{kuhn1955hungarian}. 

\subsection{Experimental Results}
\begin{figure}[htbp]
\centering
\includegraphics[width=\linewidth]{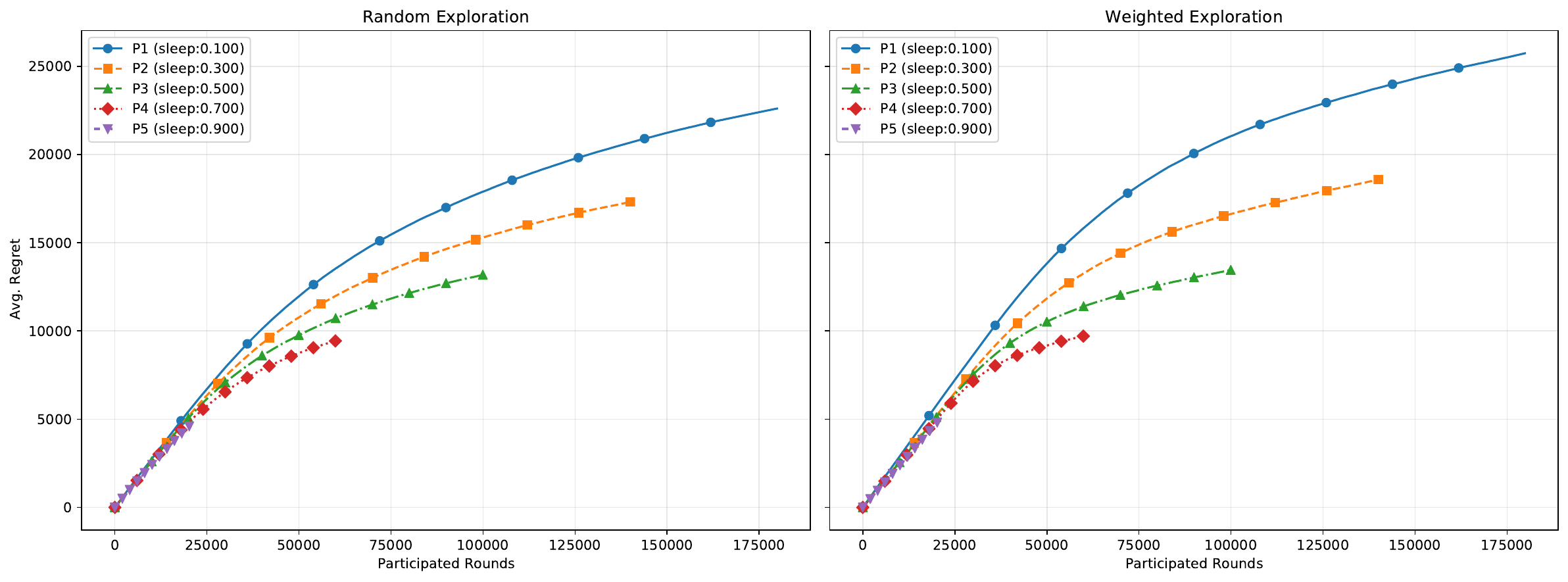}
\caption{Regret comparison between random and weighted exploration with heterogeneous player unavailability probabilities.}
  \label{fig:regret_comparison1}
\end{figure}

\begin{figure}[htbp]
\centering
\includegraphics[width=\linewidth]{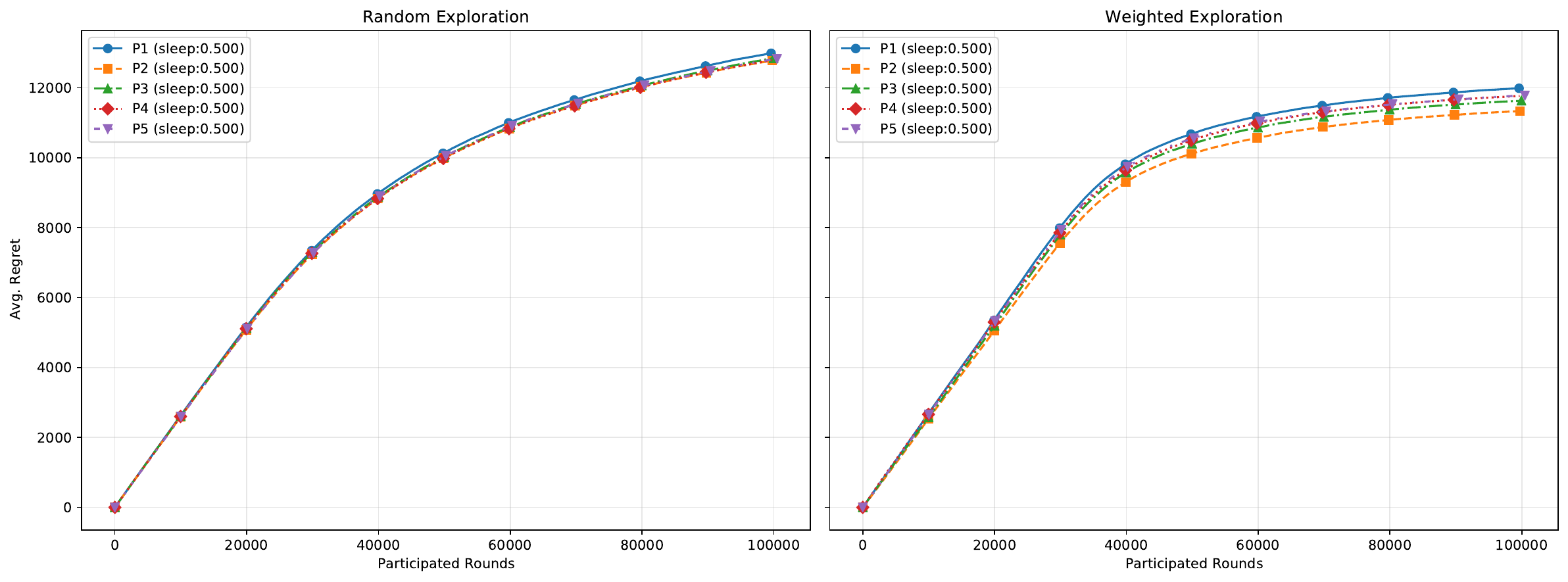}
\caption{Regret comparison between random and weighted exploration with identical player unavailability probabilities.}
\label{fig:regret_comparison2}
\end{figure}

\begin{figure}[htbp]
\centering
\includegraphics[width=\linewidth]{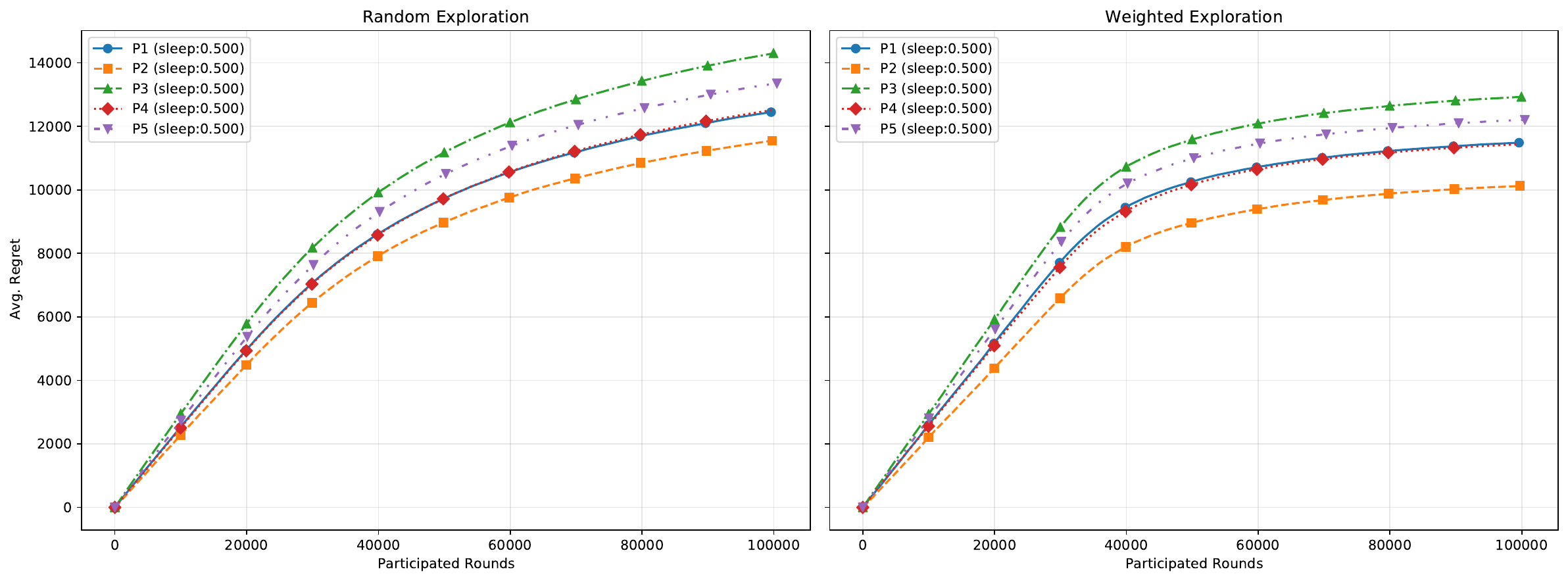}
\caption{Regret comparison between random and weighted exploration with identical player unavailability probabilities, where the preferences of arms are fixed.}
\label{fig:regret_comparison3}
\end{figure}

\textbf{Fig.~\ref{fig:regret_comparison1}.}\quad Empirical evaluation of the proposed AC-ETGS algorithms in the Sleeping Competing Bandits framework ($N=5$ players, $K=10$ arms, horizon $T=2 \times 10^5$). The evaluation is conducted over 50 distinct problem instances, with 50 independent trials performed for each instance. For each instance, the expected rewards $\mu_{i, j}$ are set to $K$ linearly spaced values in the interval $[0.1, 0.9]$ and randomly permuted across the arms. For each arm, its preference order over the players is determined randomly.
The unavailability probabilities are linearly spaced in $[0.1, 0.9]$ for both players and arms. Notably, within each instance, the environment parameters (expected rewards, arm preferences, and sleeping probabilities) and the temporal sequence of player/arm availability are generated once and fixed across its 50 trials. Thus, the variance within an instance arises solely from the stochastic reward realizations and the algorithms' internal randomness. The final reported results are obtained by first averaging over the 50 trials for each instance, and then averaging these outcomes across all 50 instances.
\vspace{\baselineskip}\\
\textbf{Fig.~\ref{fig:regret_comparison2}.}\quad Regret comparison under the same setting as Fig.~\ref{fig:regret_comparison1}, except the unavailability probabilities for all five players are fixed at $0.5$.
\vspace{\baselineskip}\\
\textbf{Fig.~\ref{fig:regret_comparison3}.}\quad Regret comparison under the same setting as Fig.~\ref{fig:regret_comparison2}, with the exception that the preference orders of arms over players is randomly initialized and kept constant across all instances and trials. 

Comparing Fig.~\ref{fig:regret_comparison2} and Fig.~\ref{fig:regret_comparison3} reveals that varying arm preferences across instances in Fig.~\ref{fig:regret_comparison2} effectively averages out the differences among players, rendering them more symmetric.
\end{document}